\documentclass{article} % For LaTeX2e
\usepackage{iclr2027_conference,times}

\usepackage{amsmath,amsfonts,bm}

\def\eqref#1{equation~\ref{#1}}
\def\1{\bm{1}}

\DeclareMathAlphabet{\mathsfit}{\encodingdefault}{\sfdefault}{m}{sl}
\SetMathAlphabet{\mathsfit}{bold}{\encodingdefault}{\sfdefault}{bx}{n}

\usepackage{amsmath}     % 若 math_commands 已载入可省
\usepackage{amssymb}     % ← 加这一行
\usepackage{amsthm}
\usepackage{booktabs}
\newtheorem{theorem}{Theorem}
\newtheorem{proposition}{Proposition}
\newtheorem{lemma}{Lemma}

\usepackage{hyperref}    % 最后
\usepackage{url}
\usepackage{booktabs}
\usepackage{array}
\usepackage{longtable}
\usepackage{booktabs,amssymb,tabularx}
\usepackage{graphicx}
\usepackage{wrapfig}
\usepackage{tabularx}
\usepackage{capt-of}
\usepackage{tabularx}

\title{PhysDEM: Physics-Defined Energy-Matching Diffusion for Spatiotemporal Field Generation under Scarce Measurements}

\iclrfinalcopy

\author{Zhenyu Liang \\
HKUST\\
\And
Yining Huang \\
UT Austin\\
\And
Yubo Zhao \\
HKUST\\
\And
Jack C.P. Cheng \\
HKUST\\
}

\begin{document}

\maketitle
\lhead{Preprint} % 必须放在 \maketitle 后，覆盖终稿页眉

% \begin{abstract}

% (DON'T READ! JUST occupy the space) Inferring spatiotemporal physical fields from a handful of costly measurements is essential for understanding complex systems and supporting decisions under uncertainty. In this regime, many fields can explain the same observations, making a distribution over plausible solutions more informative than a single prediction. Previous methods often rely on large collections of observed or simulated fields, or primarily seek a single reconstruction, limiting their suitability for uncertainty quantification when measurements are scarce. We introduce PhysDEM, a diffusion sampler that combines governing equations with scarce measurements without requiring an external dataset of complete fields. Its construction follows three stages: calibrate, train, and sample. First, an observation-conditioned Gaussian reference distribution is reweighted by a physical residual penalty to define the target distribution. Second, we show that denoising reduces to learning how the physical penalty shifts a Gaussian mean, and estimate this displacement using locally linearised physics and importance weighting. Third, a probability-flow sampler uses the learned displacement to transform reference draws into an ensemble of plausible physical fields. Additional linear Gaussian measurements are incorporated by updating the reference 

% \end{abstract}

\begin{abstract}
Generating and predicting spatiotemporal physical fields from scarce measurements is challenging, as observations are insufficient to characterize a distribution over complete fields. This limits conventional data-driven diffusion models that rely on full-field datasets. We introduce \textbf{PhysDEM}, a physics-defined diffusion framework that combines governing equations with spatially sparse observations to generate multiple plausible fields. First, we construct a Gibbs target by reweighting a measurement-conditioned Gaussian reference with PDE residual energy. Second, we derive an exact conditional-mean identity that reduces denoising to supervised learning of the standardized energy-induced mean correction. Third, a physics-displacement probability flow cancels Gaussian reference terms and enables amortized sampling with changing measurements through Gaussian conditioning, without retraining. Experiments on synthetic PDE systems and real-world-informed applications demonstrate that PhysDEM supports coherent field recovery and efficient sampling while maintaining stable diagnostics under tested noise levels, illustrating its practical value for field assessment. To our knowledge, PhysDEM is the \emph{first} physics-defined diffusion model enabling amortized spatiotemporal field inference without preassembled full-field datasets.
\end{abstract}

% To our knowledge, PhysDEM is the first physics-defined diffusion model, in which the governing equations alone define the target distribution and generate the training fields.
% By defining the target through physics and generating its own training fields, PhysDEM enables uncertainty-aware field inference without preassembled full-field datasets.

% \section{Introduction}
% \label{sec:introduction}

% Measurements at a few locations may be consistent with many full spatiotemporal fields. PhysDEM uses these measurements to calibrate the model and constrain a probability distribution over the possible fields, whose relative plausibility reflects both the physical model and the measurements. Sampling from this distribution generates plausible fields across the entire domain without requiring an external dataset of complete fields.

\section{Introduction}
\label{sec:intro}

Generating and predicting spatiotemporal physical fields is essential for scientific modeling. To address this task, generative artificial intelligence uses neural networks to learn probability distributions over plausible field realizations and generate new fields by sampling from these distributions. Diffusion models, in particular, generate samples from high-dimensional distributions through iterative denoising~\citep{song2021score}. Their key strengths include training stability, sample fidelity, distributional coverage~\citep{dhariwal2021diffusion}, and flexible conditioning~\citep{song2021score,rozet2023sda}, supporting the generation of spatially detailed field ensembles that represent diverse temporal evolutions and their uncertainty~\citep{price2025gencast,rozet2023sda}.

Most existing diffusion models for physical-field generation are \emph{data-driven}. Some learn physical patterns implicitly from extensive reanalysis or simulation datasets~\citep{price2025gencast,rozet2023sda}, whereas others explicitly incorporate governing partial differential equations (PDEs) through physics-informed training losses~\citep{bastek2025physics,wang2026fundiff} or sampling guidance~\citep{huang2024diffusionpde,jacobsen2025cocogen}. Both strategies typically rely on large collections of complete fields for training. This dependence limits their applicability when the available data consist only of scarce measurements, with time series recorded at a handful of spatial locations. These scarce measurements alone are insufficient to characterize a distribution over complete fields.

% To address this limitation, we propose a \emph{physics-driven} diffusion framework for spatiotemporal field generation based on the principle that \textbf{physics defines the target distribution, and measurements provide constraints}. We use the governing physical equations to construct a target distribution for denoiser training, without requiring an external dataset of complete fields. At inference, scarce measurements collected at a handful of spatial locations condition the generation of complete fields. This framework uses limited local observations to estimate field values in unobserved regions and explore a range of plausible evolutions of complete spatiotemporal fields.

To address this limitation, we propose a \emph{physics-driven} diffusion framework for spatiotemporal field generation, in which \textbf{physics defines the target distribution, and measurements provide constraints}. We use the governing equations to define a PDE residual energy and construct a target distribution from which we sample training fields, without requiring a preassembled full-field dataset from real-world measurements or external simulators. Once trained, the model serves as an amortized sampler, using scarce measurements at a handful of spatial locations to rapidly estimate unobserved field values and explore diverse plausible evolutions of complete spatiotemporal fields.

% To address this limitation, we propose a \emph{physics-driven} diffusion framework for spatiotemporal field generation, in which \textbf{physics defines the target distribution, and measurements provide constraints}. PDE residuals and an assumed Gaussian reference define the target from which we generate training fields, without requiring a full-field dataset from real-world measurements or external simulators. Once trained, the model serves as an amortized sampler, using scarce measurements at a handful of spatial locations to rapidly estimate unobserved field values and explore diverse plausible evolutions of complete spatiotemporal fields.

\begin{table}[t]
\centering
\caption{Comparison of spatiotemporal physical generation methods.
P: explicit use of prescribed PDE constraints in generative modeling;
D: training without requiring a full-field dataset obtained from measurements or forward simulations;
R: reuse across different numbers of measurements without retraining;
S: full-field inference from spatially sparse measurements.}
\label{tab:related-work}
\begin{tabularx}{\linewidth}{@{}Xcccc@{}}
\toprule
Method family and representative & P & D & R & S \\
\midrule
Data-trained GANs
(e.g., DGMR~\citep{ravuri2021skilful})
& $\times$ & $\times$ & $\times$ & $\times$ \\

Physics-informed GANs
(e.g., PI-GAN~\citep{yang2020physics})
& $\checkmark$ & $\checkmark$ & $\times$ & $\checkmark$ \\

Simulation-based flow matching
(e.g., FNOPE~\citep{moss2025fnope})
& $\times$ & $\times$ & $\checkmark$ & $\checkmark$ \\

Data-trained diffusion
(e.g., DPS~\citep{chung2023dps})
& $\times$ & $\times$ & $\checkmark$ & $\checkmark$ \\

Physics-loss diffusion
(e.g., PIDM~\citep{bastek2025physics})
& $\checkmark$ & $\times$ & $\times$ & $\times$ \\

Physics-guided diffusion
(e.g., DiffusionPDE~\citep{huang2024diffusionpde})
& $\checkmark$ & $\times$ & $\checkmark$ & $\checkmark$ \\

Learned-physics diffusion
(e.g., PHDME~\citep{tan2026phdme})
& $\times$ & $\times$ & $\times$ & $\times$ \\

Energy-based diffusion samplers
(e.g., iDEM~\citep{akhoundsadegh2024idem})
& $\times$ & $\checkmark$ & $\times$ & $\times$ \\
\midrule
\textbf{Physics-defined diffusion (PhysDEM, ours)}
& $\checkmark$ & $\checkmark$ & $\checkmark$ & $\checkmark$ \\
\bottomrule
\end{tabularx}
\end{table}

Specially, we introduce \textbf{PhysDEM}, \underline{\textbf{Phys}}ics-\underline{\textbf{D}}efined \underline{\textbf{E}}nergy-\underline{\textbf{M}}atching Diffusion, for spatiotemporal field generation under scarce measurements. Our mathematical derivation links the PDE residual energy to target distribution construction, denoiser training, and field sampling. In this framework, we define a \textbf{Gibbs target distribution} by reweighting a Gaussian reference that encodes assumed field variability with the PDE residual energy that promotes physical consistency. We then derive an \textbf{energy-determined conditional-mean diffusion denoiser} that implements energy matching by learning the remaining energy-induced mean correction, aligning denoiser training with the physics-defined target. At inference, scarce measurements update the Gaussian reference and denoiser inputs. The denoiser's predicted mean correction drives our \textbf{physics-displacement probability-flow sampler}, enabling amortized sampling of plausible complete spatiotemporal fields across measurement configurations without retraining. Our main contributions are as follows:

\begin{itemize}

\item We propose a physics-defined and measurement-constrained Gibbs target distribution that enables denoiser training on internally sampled complete fields without requiring a preassembled full-field dataset from real-world measurements or external simulators.

\item We mathematically derive an energy-determined conditional-mean diffusion denoiser through an exact conditional-mean identity, establishing energy matching as supervised learning of the energy-induced mean correction relative to an analytic Gaussian estimate.

\item We develop an amortized physics-displacement probability-flow sampler that analytically cancels Gaussian reference terms and uses the learned normalized correction as its velocity, accommodating changing measurements through Gaussian conditioning without retraining.
\end{itemize}

\section{Related Work}
\label{sec:related-work}

\paragraph{Spatiotemporal physical field generation.}
Generative artificial intelligence models support spatiotemporal field generation by learning distributions over plausible field realizations. GANs such as DGMR generate realistic precipitation sequences but rely on large radar archives for training~\citep{ravuri2021skilful}. PI-GAN incorporates governing equations and ensembles of scattered sensor measurements but remains susceptible to discriminator overfitting~\citep{yang2020physics}. FNOPE uses conditional flow matching for function-valued posterior inference but requires paired parameter samples and simulator outputs~\citep{moss2025fnope}. In comparison, diffusion offers an alternative based on non-adversarial denoising and flexible conditioning~\citep{song2021score}, iteratively refining noisy fields according to learned spatiotemporal dependencies to generate spatially detailed, temporally coherent, and diverse fields~\citep{price2025gencast,rozet2023sda,cachay2023dyffusion}.

\paragraph{Physics-constrained diffusion models.}
Physical constraints can enter diffusion models through PDE residual penalties during training~\citep{bastek2025physics} or guidance during sampling~\citep{huang2024diffusionpde,li2026dibilps}. However, these approaches generally rely on full-field training data obtained from measurements or forward simulations. A learned-physics approach~\citep{tan2026phdme} estimates dynamics from observations without fully specified governing equations, but has not been demonstrated for full-field inference from scarce measurements alone. Energy-based diffusion samplers instead learn to sample from prescribed unnormalized densities without precomputed target datasets~\citep{vargas2023denoising,akhoundsadegh2024idem}, but their original formulations do not explicitly support reuse across varying measurement counts without retraining. Compared with these state-of-the-art methods, PhysDEM provides a physics-defined diffusion framework that samples its own training fields from a PDE-based target and generates complete fields from scarce measurements while accommodating different measurement counts without retraining (see Table~\ref{tab:related-work}).

\section{Preliminaries}
\label{sec:prelim}

\subsection{Spatiotemporal physical fields and PDE residuals}
\label{sec:prelim:system}

We consider a physical system on a spatial domain $\Omega\subset\mathbb R^{n_\Omega}$ over $[0,T]$, where $n_\Omega$ is the spatial dimension and $T$ is the final time. Its possibly vector-valued state field $u(\xi,t)$, with $\xi\in\Omega$ and $t\in[0,T]$, satisfies the first-order evolution equation
\begin{equation}
\partial_t u+\mathcal N_z(u)-q_z=0.
\label{eq:pde}
\end{equation}
Here $\partial_t$ denotes differentiation with respect to physical time, $\mathcal N_z$ is a spatial differential operator, and $q_z$ is the source or forcing term. The static parameter field $z$ collects unknown, time-independent system parameters on which $\mathcal N_z$ and $q_z$ depend \citep{stuart2010inverse}. We discretize $z$ and $u$ on spatial and space--time grids, respectively, using the same symbols for their discrete representations. The complete candidate field is $x=(z,u)\in\mathbb R^d$, where $d$ is the total number of unknown scalar values. We model $z$ and $u$ jointly, with their physical coupling encoded through the PDE residual energy.

To incorporate the PDE into the target distribution, we measure how closely each candidate $x$ satisfies \eqref{eq:pde}. Let $\mathcal R_{\mathrm{full}}(x)\in\mathbb R^{n_f}$ collect the $n_f$ PDE residuals obtained by substituting $x$ into the discretized PDE at interior grid points. Each entry measures a local discrepancy and vanishes when the corresponding discrete equation is satisfied. Boundary conditions are incorporated into the discrete operators. These residuals form the basis of the PDE residual energy $E$ in Section~\ref{sec:method:tilt}.

\subsection{Scarce measurements}
\label{sec:prelim:obs}

We consider time-series measurements at only a handful of spatial locations. For measurement group $j$, we assume a linear observation model with additive Gaussian noise:
\begin{equation}
y_j=A_jx+\eta_j,
\qquad
\eta_j\sim\mathcal N(0,\Gamma_j).
\label{eq:obs}
\end{equation}
Here $y_j\in\mathbb R^{n_j}$ contains the $n_j$ scalar measurements in group $j$, and $A_j\in\mathbb R^{n_j\times d}$ extracts or averages the corresponding field values. The covariance matrices $\Gamma_j$ are positive definite, and the noise vectors are mutually independent and independent of the field. We collect the measurement operators, values, and noise covariances in the configuration $\mathcal C=\{(A_j,y_j,\Gamma_j)\}_j$. This model enables the closed-form Gaussian conditioning \citep{rasmussen2006gaussian} used in Section~\ref{sec:method:ref}.

\section{Methodology}
\label{sec:method}

% 导言区已有 graphicx 时，无需重复添加：
% \usepackage{graphicx}

\begin{figure}[t]
\centering
\includegraphics[width=\textwidth]
{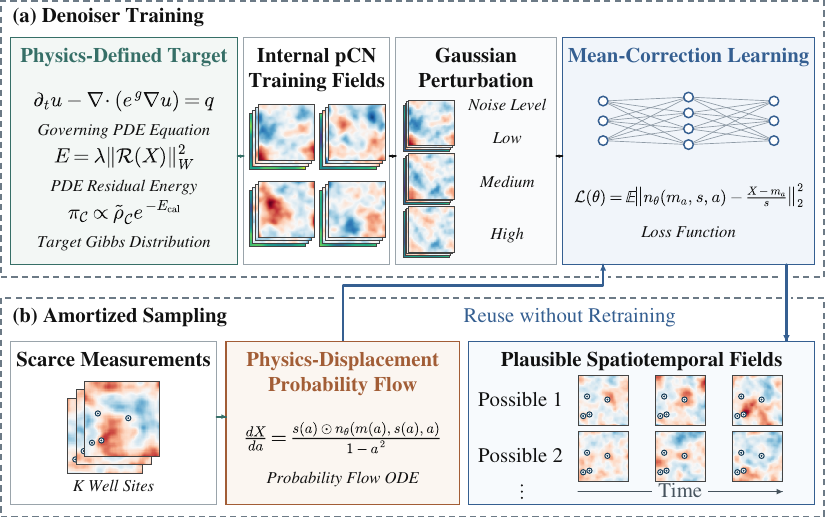}
\caption{Overview of our proposed PhysDEM, a physics-defined diffusion framework for spatiotemporal field generation under scarce measurements.}
\label{fig:physdem-framework}
\end{figure}

PhysDEM comprises three stages: \emph{definition} of the target
distribution, \emph{training} of a mathematically derived denoiser,
and amortized \emph{sampling} of fields conditioned on inference
measurements. 
Sections~\ref{sec:method:measure}--\ref{sec:method:sample}
present them in order. The overall framework is illustrated in Figure~\ref{fig:physdem-framework}.

% PhysDEM comprises three stages: \emph{definition} of the target distribution, \emph{training} of a mathematically derived denoiser, and amortized \emph{sampling} of fields conditioned on inference measurements. The PDE residual energy links these stages. Sections~\ref{sec:method:measure}--\ref{sec:method:sample} present them in order.

\subsection{Physics-defined and measurement-constrained target distribution}
\label{sec:method:measure}

We define a Gibbs target distribution by reweighting a measurement-conditioned Gaussian reference (Section~\ref{sec:method:ref}) using the PDE residual energy $E$ (Section~\ref{sec:method:tilt}), thereby encoding physical consistency and measurement constraints in a distribution over complete fields.

\subsubsection{Measurement-conditioned Gaussian reference}
\label{sec:method:ref}

We first specify a Gaussian base distribution $\rho_0=\mathcal N(\mu_0,C_0)$ for $x$, encoding assumed field means, variability, and correlations. Its positive-definite covariance $C_0=\operatorname{blkdiag}(C_z,C_u)\succ0$ has blocks $C_z$ and $C_u$ describing spatial correlations in $z$ and spatiotemporal correlations in $u$, respectively. This block structure makes $z$ and $u$ independent under $\rho_0$ (Appendix~\ref{app:covariance-specification}).

For any measurement configuration $\mathcal C$, we condition $\rho_0$ using the linear Gaussian model in \eqref{eq:obs} \citep{bishop2006pattern}. The configuration is instantiated with simulated measurements during training and actual measurements at inference. Let $A$ and $y$ stack the measurement operators and values, and let $\Gamma=\operatorname{blkdiag}(\Gamma_j)$ be the noise covariance. With $S=AC_0A^\top+\Gamma$ denoting the covariance of the noisy measurements under $\rho_0$, the conditioned mean and covariance are

\begin{equation}
\bar\mu_{\mathcal C}
=
\mu_0+C_0A^\top S^{-1}(y-A\mu_0),
\qquad
\bar C_{\mathcal C}
=
C_0-C_0A^\top S^{-1}AC_0.
\label{eq:conditioning}
\end{equation}
These define the measurement-conditioned Gaussian reference, with correlations in $C_0$ allowing measurements to inform unmeasured locations:
\begin{equation}
\bar\rho_{\mathcal C}
=
\mathcal N(\bar\mu_{\mathcal C},\bar C_{\mathcal C}).
\label{eq:measurement-reference}
\end{equation}

\subsubsection{Physics-tilted Gibbs distribution}
\label{sec:method:tilt}

The Gaussian reference in \eqref{eq:measurement-reference} incorporates measurements but does not enforce the governing PDEs. To reweight it for physical consistency, we define the PDE residual energy
\begin{equation}
E(x)
=
\lambda\|\mathcal R(x)\|_W^2
=
\lambda\mathcal R(x)^\top W\mathcal R(x).
\label{eq:pde-residual-energy}
\end{equation}
Here $\lambda\ge0$ controls the penalty strength, and $W\in\mathbb R^{n_r\times n_r}$ is a fixed positive-definite weighting matrix. The vector $\mathcal R(x)\in\mathbb R^{n_r}$ collects selected, rescaled combinations of $\mathcal R_{\mathrm{full}}(x)$ and discrepancies from known initial values when available. Our implementation uses selected space--time Fourier components for the interior residuals (Appendices~\ref{app:residual-energy} and~\ref{app:fourier-residual}).

Reweighting the Gaussian reference by $e^{-E(x)}$ gives the Gibbs target:
\begin{equation}
\pi_{\mathcal C}(x)
=
Z_{\mathcal C}^{-1}\bar\rho_{\mathcal C}(x)e^{-E(x)},
\label{eq:tilt}
\end{equation}
where $Z_{\mathcal C}$ normalizes the density. The factor $e^{-E(x)}$ assigns smaller weights to fields with larger residual energy.

We calibrate $\lambda$ against a prescribed relative residual tolerance and adjust the static-field component of $\mu_0$ to target a prescribed expectation of the spatial mean of $z$ (Appendices~\ref{app:penalty-calibration} and~\ref{app:scalar-calibration}). With calibration fixed, we absorb a quadratic approximation of $E$ into the Gaussian reference. We linearize $\mathcal R$ at a global expansion point $p_0\in\mathbb R^d$, fitted without the measurement likelihood, and let $J_0\in\mathbb R^{n_r\times d}$ denote its Jacobian there. The energy $E_{\mathrm{lin}}(x)=\lambda\|\mathcal R(p_0)+J_0(x-p_0)\|_W^2$ defines the augmented Gaussian reference $\tilde\rho_{\mathcal C}(x)\propto\bar\rho_{\mathcal C}(x)e^{-E_{\mathrm{lin}}(x)}$, with mean $\tilde\mu_{\mathcal C}$ and covariance $\tilde C_{\mathcal C}$. Defining the remaining energy as $E_{\mathrm{cal}}=E-E_{\mathrm{lin}}$ yields the equivalent representation
\begin{equation}
\pi_{\mathcal C}(x)
\propto
\tilde\rho_{\mathcal C}(x)e^{-E_{\mathrm{cal}}(x)}.
\label{eq:augmented}
\end{equation}
Subtracting $E_{\mathrm{lin}}$ avoids double counting and preserves the target in \eqref{eq:tilt}. Although $E_{\mathrm{cal}}$ may be negative, the combined density remains normalizable (Appendix~\ref{app:augmented}). Before training, we fix $\rho_0$, $E$, and $p_0$. Different measurement configurations then change the Gaussian reference and target distribution while leaving $E_{\mathrm{cal}}$ unchanged.

\subsection{Energy-determined conditional-mean diffusion denoiser}
\label{sec:method:denoiser}
\label{sec:method:training}

\paragraph{Training-field generation and Gaussian perturbations.}
To form training configurations, we sample pseudo-fields approximately from the target without measurement conditioning, whose density is proportional to $\rho_0(x)e^{-E(x)}$, and simulate scarce measurements using \eqref{eq:obs}. For each resulting configuration $\mathcal C$, preconditioned Crank--Nicolson (pCN) chains generate approximate samples from the corresponding Gibbs target $\pi_{\mathcal C}$~\citep{cotter2013mcmc}. Appendices~\ref{app:configs} and~\ref{app:pcn} describe configuration generation and target sampling, respectively.

We perturb the retained fields to construct denoising examples. For a fixed configuration, write $\tilde\rho_{\mathcal C}=\mathcal N(\tilde\mu,\tilde C)$, with $\tilde\mu=\tilde\mu_{\mathcal C}$ and $\tilde C=\tilde C_{\mathcal C}\succ0$. For the ideal construction, let $X\sim\pi_{\mathcal C}$ denote a random complete-field vector, choose a signal coefficient $a\in[0,1]$, and set $\tau=1-a^2$. With independent Gaussian noise $\xi_{\mathrm G}\sim\mathcal N(0,I_d)$, define
\begin{equation}
X_a
=
\tilde\mu+a(X-\tilde\mu)
+\sqrt{\tau}\,\tilde C^{1/2}\xi_{\mathrm G}.
\label{eq:forward}
\end{equation}
Here $X_a$ is the perturbed field, $I_d$ is the $d$-dimensional identity matrix, and $\tilde C^{1/2}$ is the symmetric positive-definite square root of $\tilde C$. The endpoints $a=1$ and $a=0$ recover the original field and an independent Gaussian reference draw, respectively. During training, we draw $a$ uniformly from $[a_{\min},a_{\max}]$, where $0\le a_{\min}<a_{\max}<1$, independently of the field and Gaussian noise. If $X$ were drawn from $\tilde\rho_{\mathcal C}$, the perturbation would preserve this reference and, for $a<1$, give the conditional mean $m_a=\tilde\mu+a(X_a-\tilde\mu)$ and covariance $\Sigma_a=\tau\tilde C$ (Lemma~\ref{lem:forward} in Appendix~\ref{app:reference-preservation}). We use $m_a$ as the analytic Gaussian denoising estimate.

\paragraph{Deriving the energy-determined conditional-mean denoiser.}
For $m\in\mathbb R^d$ and $\Sigma\in\mathbb R^{d\times d}$ with $\Sigma\succ0$, define the energy-weighted mean
\begin{equation}
D(m,\Sigma)
=
\frac{
\mathbb E_{X\sim\mathcal N(m,\Sigma)}
\left[Xe^{-E_{\mathrm{cal}}(X)}\right]
}{
\mathbb E_{X\sim\mathcal N(m,\Sigma)}
\left[e^{-E_{\mathrm{cal}}(X)}\right]
}.
\label{eq:denoiser}
\end{equation}
The denominator normalizes the energy weights. We assume it is positive and finite and that the weighted first absolute moments are finite. The following theorem identifies $D$ with the conditional mean of the original field, which is the optimal denoising estimate under squared error.

\begin{theorem}[Denoiser identity]
\label{thm:denoiser}
For a fixed measurement configuration $\mathcal C$ and fixed signal coefficient $0\le a<1$, let $X\sim\pi_{\mathcal C}$ and construct $X_a$ using \eqref{eq:forward}. Then
\begin{equation}
\mathbb E[X\mid X_a]
=
D(m_a,\Sigma_a).
\label{eq:conditional-denoiser}
\end{equation}
\end{theorem}

The identity follows because the Gibbs conditional density is the Gaussian conditional density reweighted by the remaining energy:
$p(x\mid X_a)\propto\mathcal N(x;m_a,\Sigma_a)e^{-E_{\mathrm{cal}}(x)}$.
For fixed $\mathcal C$ and $a$, integration by parts under the conditions in Appendix~\ref{app:denoiser-proof} gives
\begin{equation}
D(m_a,\Sigma_a)-m_a
=
-\Sigma_a\,
\mathbb E\!\left[
\nabla E_{\mathrm{cal}}(X)\mid X_a
\right].
\label{eq:energy-induced-correction}
\end{equation}
Thus, the PDE residual energy acts through its remainder $E_{\mathrm{cal}}$ to determine the correction to the Gaussian denoising estimate. With $E_{\mathrm{cal}}$ fixed, the same exact map $D$ applies across measurement configurations through its inputs $(m,\Sigma)$.

\paragraph{Training the diffusion denoiser.}
Evaluating \eqref{eq:denoiser} directly requires high-dimensional integration. We instead train a network $n_\theta(m,s,a)$, with parameters $\theta$, to predict the standardized mean correction. Its inputs are the Gaussian denoising estimate $m_a$, coordinate standard deviations $s=\sqrt{\operatorname{diag}\Sigma_a}$, signal coefficient $a$, and fixed grid coordinates, omitted from the notation. Here $\operatorname{diag}$ extracts the covariance diagonal, and the square root is taken elementwise. Each sampled field supplies a standardized correction label $(X-m_a)/s$, yielding the squared-error objective
\begin{equation}
\mathcal L(\theta)
=
\mathbb E\left[
\left\|
n_\theta(m_a,s,a)
-
\frac{X-m_a}{s}
\right\|_2^2
\right].
\label{eq:loss}
\end{equation}
Division by $s$ is elementwise, and the expectation averages over training configurations, retained fields, Gaussian noise, and $a$. By Theorem~\ref{thm:denoiser}, for a fixed configuration and signal coefficient, exact target draws satisfy
\begin{equation}
\mathbb E\!\left[
\frac{X-m_a}{s}
\,\middle|\,
X_a
\right]
=
\frac{D(m_a,\Sigma_a)-m_a}{s}.
\label{eq:correction-label-mean}
\end{equation}
Energy matching uses these labels to learn the standardized energy-induced mean correction without evaluating the expectations in \eqref{eq:denoiser} or \eqref{eq:energy-induced-correction}. Appendix~\ref{app:unbiased} discusses label variance and sampling assumptions. Across configurations, the network receives explicit covariance information only through its diagonal in $s$; Appendix~\ref{app:covariance-encoding} explains the resulting approximation.

Multiplying the network output by $s$ restores the correction to the original coordinate scale, yielding the denoising estimate $m_a+s\odot n_\theta(m_a,s,a)$, where $\odot$ denotes elementwise multiplication. For sampling, we normalize this correction by $\tau=1-a^2>0$:
\begin{equation}
\widetilde\Delta_\theta(m_a,s,a)
=
\frac{s\odot n_\theta(m_a,s,a)}{1-a^2}.
\label{eq:denoiser-displacement}
\end{equation}
Section~\ref{sec:method:sample} shows how this normalized correction approximates the probability-flow velocity.

\subsection{Physics-displacement probability-flow sampler}
\label{sec:method:sample}

At inference, we condition on the scarce measurements in configuration $\mathcal C$ to obtain the augmented Gaussian reference $\tilde\rho_{\mathcal C}=\mathcal N(\tilde\mu,\tilde C)$, using \eqref{eq:obs} and Section~\ref{sec:method:measure}. The base distribution $\rho_0$, residual energy $E$, expansion point $p_0$, and trained parameters $\theta$ remain fixed. The same denoiser can therefore be reused with updated Gaussian inputs, without retraining (Appendix~\ref{app:measurement-update}).

Let $p_a$ denote the density of the field $X_a$ in \eqref{eq:forward}, with score $\nabla_x\log p_a(x)$. At a candidate field $x$, write $m=\tilde\mu+a(x-\tilde\mu)$ and $\Sigma_a=(1-a^2)\tilde C$. Because the perturbation preserves the Gaussian reference, the linear drift and Gaussian score contribution cancel in the probability-flow ordinary differential equation (ODE)~\citep{song2021score}. For $0<a<1$, its exact velocity satisfies
\begin{equation}
\begin{aligned}
v_a(x)
&=
\frac{x-\tilde\mu+\tilde C\nabla_x\log p_a(x)}{a}
=
\frac{\tilde C}{a}
\nabla_x\log\!\left[
\frac{p_a(x)}{\tilde\rho_{\mathcal C}(x)}
\right]
\\
&=
\frac{D(m,\Sigma_a)-m}{1-a^2}
=
-\tilde C\,
\mathbb E\!\left[
\nabla E_{\mathrm{cal}}(X)\mid X_a=x
\right].
\end{aligned}
\label{eq:exact-flow-velocity}
\end{equation}
The final equality follows from \eqref{eq:energy-induced-correction} under the conditions in Appendix~\ref{app:denoiser-proof}; the mean-correction expression extends to $a=0$. The standardized correction learned during training becomes the learned sampling velocity after rescaling by $s$ and dividing by $1-a^2$. This connects denoiser training to approximate sampling from the same physics-defined, measurement-constrained target, without evaluating the conditional energy-gradient expectation at inference.

We initialize $X(0)\sim\tilde\rho_{\mathcal C}$ and use the learned velocity from \eqref{eq:denoiser-displacement}:
\begin{equation}
\frac{\mathrm dX(a)}{\mathrm da}
=
\widetilde\Delta_\theta(m(a),s(a),a).
\label{eq:flow}
\end{equation}
Here $m(a)=\tilde\mu+a(X(a)-\tilde\mu)$ and
$s(a)=\sqrt{(1-a^2)\operatorname{diag}\tilde C}$
are the conditional mean and coordinate standard deviations
under the Gaussian reference, respectively.
The parameter $a$ indexes denoising, while physical time $t$
remains a coordinate of the complete field. Each trajectory
therefore updates the full field jointly across space and
physical time. We integrate the learned velocity with
second-order Heun steps up to $a_{\max}<1$, generating
approximate target samples from independent reference draws.
Appendix~\ref{app:probability-flow} details the derivation
and approximation errors, while
Appendix~\ref{app:expected-coverage} provides an expected
coverage bound.

% Under suitable regularity and endpoint conditions, the distribution under the exact flow approaches $\pi_{\mathcal C}$ as $a\uparrow1$. In practice, we integrate the learned velocity using second-order Heun steps on a uniform grid ending at $a_{\max}<1$. Independent initial draws yield an ensemble of approximate target samples. Appendix~\ref{app:probability-flow} details the derivation and approximation errors, including training-chain and denoiser errors, covariance encoding, numerical integration, and endpoint truncation. Appendix~\ref{app:expected-coverage} gives an expected coverage bound for credible sets of derived quantities, accounting for model and computational discrepancies.

% Notably, although pCN is asymptotically exact under suitable ergodicity conditions~\citep{cotter2013mcmc}, direct inference requires Markov-chain computation for each new target. PhysDEM provides an amortized sampler: the offline cost of target sampling and denoiser training is shared across repeated inference queries, each using Gaussian conditioning and the learned flow without additional pCN runs. Accuracy on new measurement configurations depends on how well the trained denoiser approximates their energy-induced mean corrections.

% \newpage

% Required packages:
% \usepackage{booktabs,amsmath}

% Required packages:
% \usepackage{booktabs,amsmath}

\begin{table}[t]
\centering
\caption{Baseline diagnostics at $K=4$ with 64-member ensembles. Bold marks the best values by the stated criteria: largest SBC counts and values closest to one elsewhere.}
\label{tab:baselines}

\footnotesize
\setlength{\tabcolsep}{1.5pt}
\renewcommand{\arraystretch}{1.12}

\begin{tabular*}{\linewidth}{@{\extracolsep{\fill}}l*{5}{ccc}@{}}
\toprule
& \multicolumn{3}{c}{Darcy}
& \multicolumn{3}{c}{Adv.--diff.}
& \multicolumn{3}{c}{Fisher--KPP}
& \multicolumn{3}{c}{SPE10}
& \multicolumn{3}{c}{Bemidji} \\
\cmidrule(lr){2-4}
\cmidrule(lr){5-7}
\cmidrule(lr){8-10}
\cmidrule(lr){11-13}
\cmidrule(lr){14-16}
Method
& SBC & $z_{\mathrm{obs}}$ & $r_{\mathrm{PDE}}$
& SBC & $z_{\mathrm{obs}}$ & $r_{\mathrm{PDE}}$
& SBC & $z_{\mathrm{obs}}$ & $r_{\mathrm{PDE}}$
& SBC & $z_{\mathrm{obs}}$ & $r_{\mathrm{PDE}}$
& SBC & $z_{\mathrm{obs}}$ & $r_{\mathrm{PDE}}$ \\
\midrule

DiffusionPDE
& 3/5 & 3.99 & 1.11
& 5/7 & 8.11 & \textbf{1.02}
& 4/5 & 4.40 & \textbf{1.02}
& 1/5 & 13.47 & 0.89
& 4/7 & 9.51 & 1.39 \\

DPS
& 3/5 & 14.79 & 2.53
& 1/7 & 30.03 & 1.07
& 3/5 & 16.56 & 1.49
& 3/5 & 42.85 & 0.77
& 3/7 & 38.19 & 1.95 \\

DAPS
& 3/5 & 1.59 & 1.34
& 5/7 & 4.79 & 0.79
& 1/5 & 1.65 & 1.32
& 3/5 & 5.22 & 0.70
& 2/7 & 5.41 & 1.20 \\

PnP-DM
& 2/5 & 3.25 & 3.28
& 4/7 & 3.81 & 1.49
& 1/5 & 3.82 & 1.95
& 1/5 & 5.37 & 0.88
& 5/7 & 6.04 & 1.68 \\

FunDPS
& 4/5 & 0.21 & 1.39
& 3/7 & 0.24 & 3.31
& \textbf{5/5} & 0.54 & 4.92
& 3/5 & 1.87 & 0.79
& 3/7 & 0.01 & 1.64 \\

PIDM
& 3/5 & 0.81 & 1.54
& 2/7 & 0.91 & 1.18
& 4/5 & 0.83 & 0.68
& 4/5 & 2.26 & 0.73
& 3/7 & 0.80 & 0.85 \\

Di-BiLPS
& 2/5 & 5.82 & 1.50
& 4/7 & 2.70 & 0.52
& 0/5 & 1.92 & 7.39
& 0/5 & 131.47 & 0.54
& 3/7 & 1.80 & 26.29 \\

FNOPE
& 3/5 & 6.57 & 0.00
& \textbf{6/7} & 10.73 & 0.00
& 1/5 & 8.39 & 0.00
& 2/5 & 53.65 & 0.00
& 4/7 & 10.48 & 0.00 \\

\midrule
\textbf{PhysDEM}
& \textbf{5/5} & \textbf{1.00} & \textbf{1.01}
& \textbf{6/7} & \textbf{0.99} & 1.06
& \textbf{5/5} & \textbf{1.09} & 1.12
& \textbf{5/5} & \textbf{0.99} & \textbf{1.00}
& \textbf{7/7} & \textbf{0.96} & \textbf{1.09} \\
\bottomrule
\end{tabular*}
\end{table}

% Preamble, if not already loaded:
% \usepackage{graphicx}

\begin{figure}[!t]
    \centering
    \includegraphics[width=\textwidth]{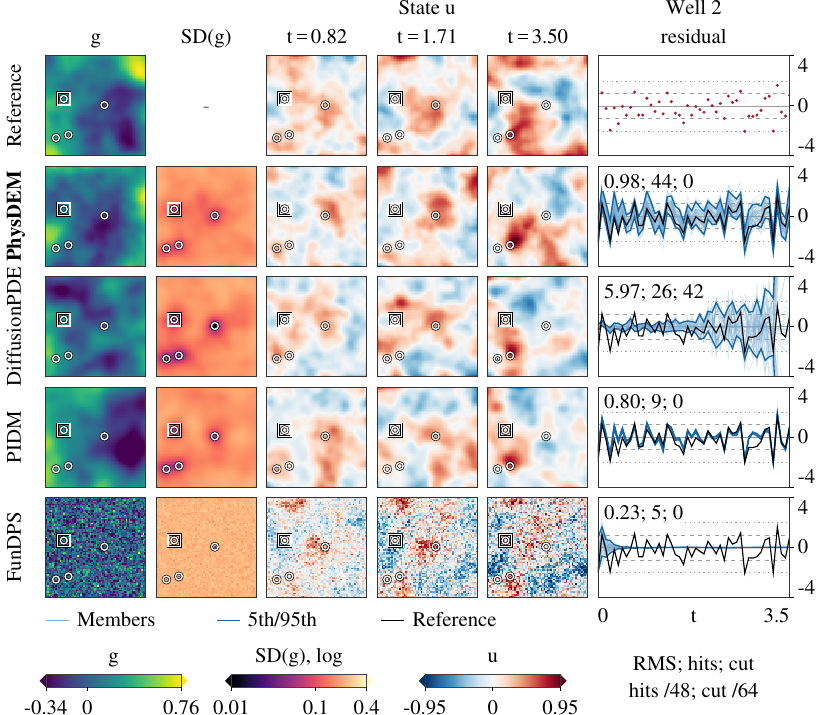}
    \caption{Darcy baseline comparison of fields and observation consistency at $K=4$.}
    \label{fig:darcy-main}
\end{figure}

\section{Experiments and Results}
We evaluate PhysDEM on three synthetic PDE systems (Darcy, advection--diffusion, and Fisher--KPP) and two real-world-informed applications (SPE10 and Bemidji). We compare field recovery from scarce measurements against baselines and examine measurement reuse, ablations, sampling cost, and noise sensitivity. Further case details and data sources are documented in Appendix~\ref{app:exp-setup}.

\begin{figure}[t]
\centering
\includegraphics[width=\textwidth]{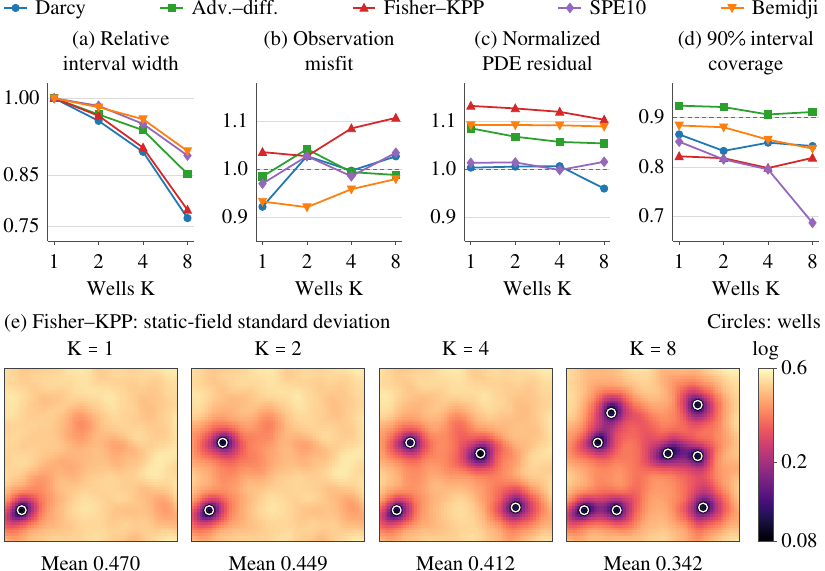}
\caption{Reuse across measurement counts without retraining.}
\label{fig:measurement-reuse}
\end{figure}

\begin{wraptable}{r}{0.50\textwidth}
\centering
\caption{Component ablations and tolerance sensitivity on Darcy
($K=4$, 64-member ensembles).}
\label{tab:ablation}

\setlength{\tabcolsep}{1.5pt}
\begin{tabularx}{\linewidth}{
@{}>{\raggedright\arraybackslash}Xcccc@{}
}
\toprule
Variant
& \shortstack{SBC\\$\max|z|$}
& $z_{\mathrm{obs}}$
& $r_{\mathrm{PDE}}$
& $\mathrm{SSR}_g$ \\
\midrule
Without reference augmentation
& $11.65\,\times$ & 1.05 & 0.62 & 1.53 \\

$\tau$-scaled labels
& $11.84\,\times$ & 2.07 & 0.97 & 0.75 \\

Augmented reference only
& -- & 0.99 & 0.96 & 0.54 \\
\midrule
Tolerance $0.05$
& $4.52\,\times$ & 0.98 & 0.35 & 0.88 \\

Tolerance $0.15$
& $5.20\,\times$ & \textbf{1.00} & 2.10 & \textbf{1.08} \\
\midrule
\textbf{PhysDEM (full)}
& $\mathbf{2.18}\,\checkmark$
& \textbf{1.00}
& \textbf{1.01}
& 0.89 \\
\bottomrule
\end{tabularx}
\end{wraptable}

% 在这里直接接正文，文字会在表格左侧环绕。

\paragraph{Baseline comparison.} Table~\ref{tab:baselines} shows that PhysDEM attains the highest or tied-highest reported SBC counts across five cases, with observation misfits of $0.96$--$1.09$ and selected residual ratios of $1.00$--$1.12$. Four cases pass all reported mean-rank checks. Its observation misfit is closest to the noise scale in each of the five test cases. This balance is achieved using internally generated training fields, without a preassembled external full-field dataset. Figure~\ref{fig:darcy-main} shows coherent spatial structure and temporal evolution under scarce measurements. At the displayed well, PhysDEM's 90\% intervals contain the reference at $44/48$ times. DiffusionPDE shows late-time mismatch ($26/48$ hits), PIDM underdispersed intervals ($9/48$), and FunDPS noisy fields despite tight observation fitting ($5/48$). Protocols and additional comparisons appear in Appendix~\ref{app:exp-baselines}. Figure~\ref{fig:bemidji-site-informed} separately maps the 1995 mean and spread of log benzene concentration, identifying regions of higher predicted concentration and greater ensemble variability. This illustrates PhysDEM's practical value for field assessment in a real-world-informed setting; comparison with 147 records used in reference construction remains descriptive.

\paragraph{Reuse across measurements.} Figure~\ref{fig:measurement-reuse} shows that PhysDEM adapts to $K=1,2,4,8$ wells using one fixed network per case. From one to eight wells, full-field mean pointwise 90\% interval widths decrease by $10.3$--$23.5\%$ across all five cases, while observation and residual diagnostics remain near their reference scales. Fisher--KPP's mean static-field standard deviation falls from $0.470$ to $0.342$, with lower spread around measured wells. These spatial responses show how scarce measurements reshape the generated ensembles. Measurement changes enter through the conditioned Gaussian reference, allowing the learned denoiser to be reused without retraining. This supports repeated inference across the tested well counts. Protocols and the SPE10 extension are in Appendix~\ref{app:exp-measurement-reuse}.

% 导言区已有这些宏包时，无需重复添加：
% \usepackage{graphicx}
% \usepackage{wrapfig}

% 页面顶部：通栏图
\begin{figure}[t]
\centering
\includegraphics[width=\textwidth]
{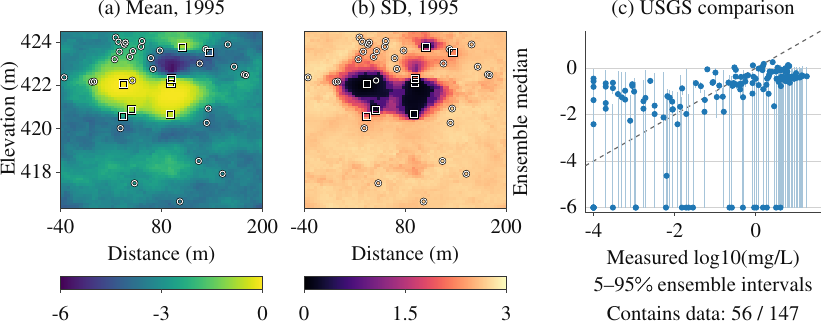}
\caption{Site-informed Bemidji fields and descriptive comparison with 147 USGS records.}
\label{fig:bemidji-site-informed}
\end{figure}

% 页面右侧：两张图共用一个环绕区域，分别编号
\begin{wrapfigure}{r}{0.50\textwidth}
\centering
\includegraphics[width=\linewidth]
{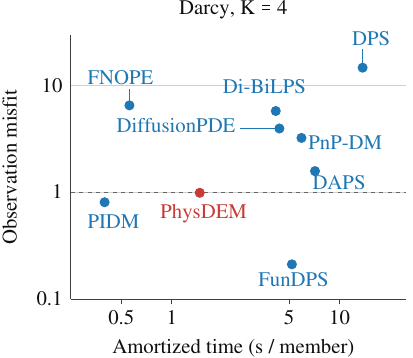}

% 仅增加第一张图 caption 与下一张图之间的间隔
\begingroup
\setlength{\belowcaptionskip}{15pt}
\caption{Amortized sampling cost and observation misfit
on Darcy at $K=4$.}
\label{fig:sampling-cost}
\endgroup

\includegraphics[width=\linewidth]
{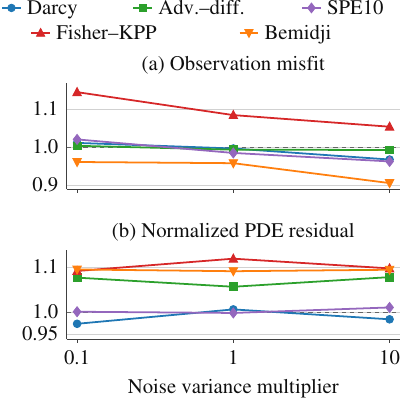}
\caption{Sensitivity to observation-noise variance
using a fixed checkpoint for each case.}
\label{fig:noise-sensitivity}
\end{wrapfigure}

% 在这里紧接正文，文字将在两张小图左侧环绕。

\paragraph{Ablation study.} In Table~\ref{tab:ablation}, training without reference augmentation or using $\tau$-scaled labels raises the largest mean-rank discrepancy from $2.18$ to $11.65$ or $11.84$. These components absorb linearized physics and normalize correction scales. The learned velocity improves the spread--skill balance, raising $\mathrm{SSR}_g$ from $0.54$ to $0.89$. Diagnostics involve trade-offs: tolerance $0.15$ moves SSR closer to one ($1.08$), but fails the rank check and roughly doubles the residual ratio. PhysDEM thus balances the reported diagnostics; see Appendix~\ref{app:ablation} for more details.

\paragraph{Sampling time and noise sensitivity.} Figure~\ref{fig:sampling-cost} shows that PhysDEM takes an amortized $1.47$ s/member on Darcy, faster than six of eight baselines including DiffusionPDE ($4.38$ s/member), with observation misfit closest to one. Figure~\ref{fig:noise-sensitivity} shows observation and residual diagnostics near one in five cases for noise variance multipliers $0.1,1,10$. The learned flow amortizes field inference, while Gaussian conditioning adapts to noise without retraining, supporting repeated inference (Appendix~\ref{app:exp-cost-noise}).

\section{Conclusion and Limitations}
We develop PhysDEM for spatiotemporal field generation from scarce measurements without external full-field datasets. Extensive experiments show coherent field recovery, efficient sampling, and stable diagnostics under tested noise levels. Gaussian conditioning enables amortized sampling with changing measurements without retraining, while site-informed cases illustrate real-world value. Weakly informed static fields retain broad uncertainty, limiting unique recovery. Future work will use richer measurements to improve joint static and dynamic field recovery.

\par
\WFclear
\clearpage
\newpage

\newpage

\bibliography{iclr2027_conference}
\bibliographystyle{iclr2027_conference}

\newpage

\appendix

\section{Supporting preliminaries}
\label{app:prelim-support}

This appendix states the conventions and assumptions used in Sections~\ref{sec:prelim}--\ref{sec:method:sample} and derives the Gaussian identities underlying target construction, denoiser training, and sampling.

\subsection{Notation and mathematical tools}
\label{app:notation}

Write $x=(z,u)\in\mathbb R^d$, where $z\in\mathbb R^{d_z}$, $u\in\mathbb R^{d_u}$, and $d=d_z+d_u$. The vectors collect the discretized static parameter and state fields, including all modeled state components. Physical time $t$ is a coordinate of $u$; the signal coefficient $a$ and auxiliary diffusion time $v$ parameterize denoising. A measurement group is a block of observations, such as a sensor time series. The configuration $\mathcal C=\{(A_j,y_j,\Gamma_j)\}_j$ specifies both the observation setup and its observed values.

\begin{center}
\renewcommand{\arraystretch}{1.08}
\begin{tabular}{@{}p{0.31\linewidth}p{\dimexpr0.69\linewidth-2\tabcolsep\relax}@{}}
\toprule
\textbf{Symbol} & \textbf{Meaning} \\
\midrule
$\mathcal R_{\mathrm{full}},\mathcal R_{\mathrm{ic}}$ & Interior PDE and initial-value residuals. \\
$B,\mathcal R,W$ & Interior selection and scaling, combined residual vector, and positive-definite residual metric. \\
$h,E,\lambda$ & Residual score, PDE residual energy, and penalty strength. \\
$\rho_0,\bar\rho_{\mathcal C},\tilde\rho_{\mathcal C}$ & Base, measurement-conditioned, and augmented Gaussian references. \\
$\pi_{\mathcal C},Z_{\mathcal C}$ & Gibbs target and normalizer relative to $\bar\rho_{\mathcal C}$. \\
$p_0,J_0,E_{\mathrm{lin}},E_{\mathrm{cal}}$ & Fixed global expansion point, residual Jacobian, quadratic energy, and remaining energy. \\
$p,J_p,E_{\mathrm{lin}}^{(p)},Q_J$ & Local expansion point, Jacobian, quadratic energy, and Gaussian proposal reference for training chains. \\
$a,\tau$ & Signal coefficient and noise covariance multiplier. \\
$X,X_a,X(a)$ & Random complete field, perturbed field, and sampling trajectory. \\
$m,\Sigma,D$ & Gaussian query mean, covariance, and energy-weighted mean. \\
$\Delta,\widetilde\Delta$ & Mean correction and its normalization by $\tau$. \\
$s,n_\theta,\widetilde\Delta_\theta$ & Coordinate standard deviations, standardized correction network, and learned sampling velocity. \\
\bottomrule
\end{tabular}
\end{center}

Distribution symbols also denote their densities. We use $M\succ0$ for symmetric positive definiteness and $M_1\preceq M_2$ when $M_2-M_1$ is positive semidefinite. The norm $\|\cdot\|_2$ is Euclidean, and $\|r\|_W^2=r^\top Wr$. Matrix square roots are symmetric unless a sampling factor is explicitly specified. The operator $\operatorname{diag}$ extracts a matrix diagonal or constructs a diagonal matrix from a vector, as determined by its argument. Square roots of variance vectors, division by $s$, and products denoted by $\odot$ are elementwise. Relative entropy $\operatorname{KL}$ uses natural logarithms.

\subsection{Assumptions}
\label{sec:app_assumption}
\label{app:assumptions}

\paragraph{Field and measurement model.}
The analysis concerns the finite-dimensional field representation in Section~\ref{sec:prelim:system}. Parameter and state fields are modeled jointly, with their physical coupling specified by the residual energy rather than by eliminating $u$ through a forward PDE solve. Prescribed boundary conditions enter the discrete operators, and available initial values enter the residual vector defined in Appendix~\ref{app:residual-energy}. The observation model is linear in the modeled coordinates $x$, with positive-definite noise covariances as in \eqref{eq:obs}. In particular, parameter transformations are part of these coordinates: an observation of a log-parameter can be linear in that log-parameter, whereas an observation of the corresponding untransformed coefficient generally is not. The same additive Gaussian noise assumption must hold in the coordinates used by the observation model.

\paragraph{Fixed target family.}
The base covariance satisfies $C_0\succ0$. After calibration, the base distribution $\rho_0$, residual construction $\mathcal R$, weighting matrix $W\succ0$, finite penalty strength $\lambda\ge0$, and global expansion point $p_0$ are held fixed. Measurement configurations may then change the Gaussian reference and target while preserving the remaining energy $E_{\mathrm{cal}}$. This is the fixed-model setting for reuse in Sections~\ref{sec:method:denoiser} and~\ref{sec:method:sample}.

\paragraph{Integrability.}
The score $h$ is measurable, nonnegative, and finite $\bar\rho_{\mathcal C}$-almost surely. Since $E=\lambda h\ge0$, the target normalizer and its polynomial moments satisfy
\begin{equation}
0<Z_{\mathcal C}=\mathbb E_{\bar\rho_{\mathcal C}}[e^{-E(X)}]\le1,
\qquad
\mathbb E_{\pi_{\mathcal C}}\|X\|_2^r
\le Z_{\mathcal C}^{-1}\mathbb E_{\bar\rho_{\mathcal C}}\|X\|_2^r<\infty,
\quad r\ge0.
\label{eq:app-a-target-integrability}
\end{equation}
Appendix~\ref{app:denoiser} establishes the corresponding integrability for the augmented Gaussian queries used in training, including when $E_{\mathrm{cal}}$ is negative. General Gaussian queries are admitted only when their normalizers and required weighted moments are finite.

\paragraph{Differentiation and numerical operations.}
The Jacobian $J(x)=\nabla\mathcal R(x)\in\mathbb R^{n_r\times d}$ acts on a field increment by $J(x)\delta x$. Residual evaluation, Jacobian actions, covariance products, and Gaussian reference draws are assumed available. The target definition requires no differentiability; derivatives, integration by parts, and endpoint limits use the additional conditions in Appendices~\ref{app:denoiser-proof} and~\ref{app:probability-flow}. Their exact identities are separated from the numerical approximations described below.

\subsection{Gaussian identities for quadratic energies}
\label{app:gaussian-facts}

Quadratic energies can be incorporated analytically into a Gaussian reference. This identity supplies both the global augmentation used by the denoiser and the local references used by training chains.

\begin{theorem}[Quadratic energies]
\label{thm:quadratic}
Let $m\in\mathbb R^d$, $\Sigma\succ0$, $L\in\mathbb R^{n_r\times d}$, $b\in\mathbb R^{n_r}$, and $W\succ0$. For $\lambda>0$, the distribution proportional to $\mathcal N(x;m,\Sigma)e^{-\lambda\|Lx-b\|_W^2}$ is Gaussian. Writing $V_\lambda=(2\lambda)^{-1}W^{-1}$ and $S_L=L\Sigma L^\top+V_\lambda$, its parameters are
\begin{equation}
\begin{aligned}
\mu_{\mathrm{tilt}}
&=m+\Sigma L^\top S_L^{-1}(b-Lm),\\
\Sigma_{\mathrm{tilt}}
&=\Sigma-\Sigma L^\top S_L^{-1}L\Sigma
=(\Sigma^{-1}+2\lambda L^\top WL)^{-1}.
\end{aligned}
\label{eq:quadratic}
\end{equation}
At $\lambda=0$, the mean and covariance remain $m$ and $\Sigma$.
\end{theorem}

\begin{proof}
Completing the square in the Gaussian log density minus the quadratic energy gives
\begin{equation}
P=\Sigma^{-1}+2\lambda L^\top WL,
\qquad
\mu_{\mathrm{tilt}}=P^{-1}(\Sigma^{-1}m+2\lambda L^\top Wb),
\qquad
\Sigma_{\mathrm{tilt}}=P^{-1}.
\label{eq:app-a-completing-square}
\end{equation}
The Woodbury identity gives the covariance-update form and the corresponding mean in \eqref{eq:quadratic}.
\end{proof}

Define $G=W^{1/2}L\Sigma L^\top W^{1/2}$ and $c=W^{1/2}(Lm-b)$. For $\lambda\ge0$, the Gaussian normalization integral is
\begin{equation}
\Phi_G(\lambda;m,\Sigma)
:=\mathbb E_{X\sim\mathcal N(m,\Sigma)}
\left[e^{-\lambda\|LX-b\|_W^2}\right]
=\frac{\exp\{-\lambda c^\top(I+2\lambda G)^{-1}c\}}
{\det(I+2\lambda G)^{1/2}}.
\label{eq:normconst}
\end{equation}
This follows by diagonalizing the positive-semidefinite matrix $G$ and evaluating the resulting one-dimensional Gaussian integrals; zero eigenvalues are allowed. For $\lambda>0$, the same tilt is conditioning on a virtual linear observation with value $b$, operator $L$, and noise covariance $V_\lambda$. This observation represents the quadratic physical penalty and introduces no additional measurement data.

\section{Target definition, calibration, and Gaussian augmentation}
\label{app:calibrate}

This appendix supports Section~\ref{sec:method:measure}. The Gaussian reference specifies field variability and measurement information, while the PDE residual energy defines their physical reweighting. Calibration selects the model parameters; Gaussian augmentation then exposes an analytic part of the physics without changing the target distribution.

\subsection{Measurement-conditioned Gaussian reference}
\label{app:reference}

\subsubsection{Covariance specification}
\label{app:covariance-specification}

We use the base covariance $C_0=\operatorname{blkdiag}(C_z,C_u)\succ0$ from Section~\ref{sec:method:ref}. Mat\'ern covariance models specify the static and state blocks, with variances, correlation lengths, and smoothness supplied by the physical problem configuration. These choices encode assumed spatial and spatiotemporal structure. Independence between $z$ and $u$ applies to the base Gaussian; the residual energy couples them in the Gibbs target. The scalar calibration in Appendix~\ref{app:scalar-calibration} adjusts the base mean while keeping $C_0$ fixed.

\subsubsection{Measurement conditioning and Matheron's rule}
\label{app:measurement-conditioning}

Stacking the measurement groups gives $A$, $y$, and $\Gamma=\operatorname{blkdiag}(\Gamma_j)$. Let $S=AC_0A^\top+\Gamma$ and $G_A=C_0A^\top S^{-1}$. The conditioned parameters are given by \eqref{eq:conditioning}. Correlations in $C_0$ propagate measurement information to unmeasured locations. With the block-diagonal base, state-only measurements update the state block at this stage; physical reweighting subsequently links state and parameter fields.

Matheron's rule samples the conditional reference without factoring its full covariance~\citep{wilson2021pathwise}. Draw $X_{\mathrm p}\sim\mathcal N(\mu_0,C_0)$ and $\eta_{\mathrm p}\sim\mathcal N(0,\Gamma)$ independently, and set
\begin{equation}
X=X_{\mathrm p}+G_A(y-AX_{\mathrm p}-\eta_{\mathrm p}).
\label{eq:matheron}
\end{equation}
Its mean and covariance are
\begin{equation}
\begin{aligned}
\mathbb E[X]&=\mu_0+G_A(y-A\mu_0)=\bar\mu_{\mathcal C},\\
\operatorname{Cov}(X)
&=(I-G_AA)C_0(I-G_AA)^\top+G_A\Gamma G_A^\top
=\bar C_{\mathcal C}.
\end{aligned}
\label{eq:app-b-matheron-moments}
\end{equation}
The construction requires base draws, covariance products, and a solve with $S$, whose dimension is the number of scalar measurements, including retained time samples. Independent measurement groups may be incorporated sequentially using Appendix~\ref{app:measurement-update}.

\subsubsection{FFT-based Gaussian operations}
\label{app:fft-reference}

For stationary covariance blocks, circulant embedding enables FFT-based products and Gaussian draws~\citep{graham2018circulant}. Let $P_{\mathrm{ext}}$ restrict an extended periodic grid with $M$ points to the modeled grid. An exact embedding of a block $C$ gives
\begin{equation}
C=P_{\mathrm{ext}}C_{\mathrm{ext}}P_{\mathrm{ext}}^\top,
\qquad
X_{\mathrm{blk}}=P_{\mathrm{ext}}C_{\mathrm{ext}}^{1/2}\xi_{\mathrm G},
\qquad
\xi_{\mathrm G}\sim\mathcal N(0,I_M),
\label{eq:app-b-circulant-embedding}
\end{equation}
where $C_{\mathrm{ext}}\succeq0$. Padding, applying an FFT-based covariance product, and restricting costs $O(M\log M)$ per block. An overall $O(d\log d)$ cost requires bounded embedding overhead. The rectangular factor in \eqref{eq:app-b-circulant-embedding} is sufficient for sampling; it need not be the symmetric square root of $C$.

The same effective covariance must be used for draws and products. If spectral clipping changes an embedding, it constitutes a covariance approximation. Positive semidefiniteness of the embedding and positive-definite observation noise do not by themselves establish $C_0\succ0$; the density and precision identities require the positive-definite field covariance assumed in Appendix~\ref{app:assumptions}. The conditioning implementation uses covariance products and measurement-space solves rather than explicitly forming $C_0^{-1}$.

\subsection{Physics-tilted Gibbs distribution and calibration}
\label{app:measure}

\subsubsection{Residual energy and information projection}
\label{app:residual-energy}

Let $B\in\mathbb R^{(n_r-n_{\mathrm{ic}})\times n_f}$ select and rescale the interior residuals. Available initial values give $\mathcal R_{\mathrm{ic}}(x)=A_{\mathrm{ic}}x-b_{\mathrm{ic}}\in\mathbb R^{n_{\mathrm{ic}}}$, where $A_{\mathrm{ic}}$ extracts the corresponding initial-state coordinates and $b_{\mathrm{ic}}$ contains their prescribed values. Then
\begin{equation}
\mathcal R(x)=
\begin{bmatrix}
B\mathcal R_{\mathrm{full}}(x)\\
\mathcal R_{\mathrm{ic}}(x)
\end{bmatrix},
\qquad
h(x)=\|\mathcal R(x)\|_W^2,
\qquad
E(x)=\lambda h(x).
\label{eq:fixed-residual-model}
\end{equation}
The initial block is omitted when no initial values are prescribed; those state coordinates then remain part of the joint field. Boundary conditions are incorporated into the discrete operators. At finite $\lambda$, the retained PDE and initial-value residuals enter as soft penalties. The implementation uses a diagonal $W\succ0$. Nondimensionalization or specified relative scales account for residuals with different units; scaling may be assigned to $B$ or $W$ provided the quadratic form is unchanged.

The Gaussian reference and the PDE energy have complementary roles. For two fields with finite residual energies, their relative density changes according to
\begin{equation}
\frac{\pi_{\mathcal C}(x_1)/\pi_{\mathcal C}(x_2)}
{\bar\rho_{\mathcal C}(x_1)/\bar\rho_{\mathcal C}(x_2)}
=\exp\{-\lambda[h(x_1)-h(x_2)]\}.
\label{eq:app-b-physical-reweighting}
\end{equation}
Thus the physical model explicitly changes the relative weights supplied by the Gaussian reference. The following result further characterizes the Gibbs target as the smallest change to that reference, in relative entropy, that achieves a specified expected residual score.

For fixed $\bar\rho=\bar\rho_{\mathcal C}$ and $h$, write $\pi_\lambda(\mathrm dx)=Z_\lambda^{-1}e^{-\lambda h(x)}\bar\rho(\mathrm dx)$, with $Z_\lambda=\mathbb E_{\bar\rho}[e^{-\lambda h}]$.

\begin{proposition}[Information projection]
\label{prop:iproj}
Suppose $h\ge0$ is nonconstant under $\bar\rho$ and $\varepsilon_0=\mathbb E_{\bar\rho}h<\infty$. Set $h_{\min}=\operatorname*{ess\,inf}_{\bar\rho}h$. For each $\varepsilon\in(h_{\min},\varepsilon_0)$, there is a unique $\lambda>0$ with $\mathbb E_{\pi_\lambda}h=\varepsilon$. The distribution $\pi_\lambda$ uniquely minimizes $\operatorname{KL}(p\|\bar\rho)$ over $p\ll\bar\rho$ satisfying $\mathbb E_p h\le\varepsilon$.
\end{proposition}

\begin{proof}
For $\lambda>0$, boundedness of $h^j e^{-\lambda h}$ for $j=1,2$ justifies differentiation:
\begin{equation}
\frac{\mathrm d}{\mathrm d\lambda}\log Z_\lambda=-\mathbb E_{\pi_\lambda}h,
\qquad
\frac{\mathrm d}{\mathrm d\lambda}\mathbb E_{\pi_\lambda}h
=-\operatorname{Var}_{\pi_\lambda}(h)<0.
\label{eq:monotone}
\end{equation}
The expectation decreases continuously from $\varepsilon_0$ to $h_{\min}$, giving the unique multiplier. For every feasible $p$ with finite relative entropy,
\begin{equation}
\operatorname{KL}(p\|\bar\rho)-\operatorname{KL}(\pi_\lambda\|\bar\rho)
=\operatorname{KL}(p\|\pi_\lambda)
+\lambda(\varepsilon-\mathbb E_p h)\ge0.
\label{eq:app-b-information-projection-gap}
\end{equation}
Equality holds only for $p=\pi_\lambda$, as in the information-projection characterization~\citep{csiszar1975idivergence}.
\end{proof}

For $\varepsilon\ge\varepsilon_0$, the reference itself is optimal. The proposition explains how physics enters target definition through the selected residual score. The numerical calibration below approximates this expected-score condition for a fixed residual construction.

\subsubsection{Fourier representation of the PDE residuals}
\label{app:testspace}
\label{app:fourier-residual}

The interior part of $B$ uses a fixed set of space--time Fourier components of the governing PDE residual. Let $\mathcal F$ denote the real FFT over the spatial and physical-time axes, and let $\{x^{(\ell)}\}_{\ell=1}^{N_{\mathrm{ref}}}$ be the reference draws used when constructing the residual representation. Define the estimated spectral powers and retained frequencies by
\begin{equation}
\begin{aligned}
P(\kappa)
&=\frac{1}{N_{\mathrm{ref}}}\sum_{\ell=1}^{N_{\mathrm{ref}}}
\left|[\mathcal F\mathcal R_{\mathrm{full}}(x^{(\ell)})]_\kappa\right|^2,\\
P_{\max}&=\max_\kappa P(\kappa),
\qquad
\mathcal K=\{\kappa:P(\kappa)>P_{\max}/\mathrm{cond}\},
\end{aligned}
\label{eq:app-b-fourier-selection}
\end{equation}
where $P_{\max}>0$ and $\mathrm{cond}>1$. The selected residual coordinates are the real and imaginary parts of
\begin{equation}
c_\kappa(x)
=\frac{[\mathcal F\mathcal R_{\mathrm{full}}(x)]_\kappa}{\sqrt{P(\kappa)}},
\qquad
B\mathcal R_{\mathrm{full}}(x)
=\operatorname{stack}_{\kappa\in\mathcal K}
\begin{bmatrix}\operatorname{Re}c_\kappa(x)\\\operatorname{Im}c_\kappa(x)\end{bmatrix}.
\label{eq:app-b-fourier-coordinates}
\end{equation}
This scaling gives unit estimated second moment per retained complex coefficient and balances components with different reference amplitudes. It does not require centering or whitening the residual vector. Real and imaginary rows of a retained frequency receive a common weight. Initial-value residuals remain in their original coordinates and are combined with the spectral rows using the specified weights in $W$. The implemented common normalization makes the estimated reference mean of $h$ equal to one.

With $W=\operatorname{blkdiag}(W_{\mathrm{pde}},W_{\mathrm{ic}})$, the resulting score is
\begin{equation}
h(x)=\|B\mathcal R_{\mathrm{full}}(x)\|_{W_{\mathrm{pde}}}^2
+\|\mathcal R_{\mathrm{ic}}(x)\|_{W_{\mathrm{ic}}}^2.
\label{eq:tested_vs_physical}
\end{equation}
The calibration tolerance applies to this selected residual representation; it does not by itself bound unselected components. This agrees with the soft, selected-residual target defined in Section~\ref{sec:method:tilt}. The retained frequencies, spectral scales, and weights are part of the physical model. They must remain fixed after calibration for measurement updates to change only the Gaussian reference and preserve the residual energy.

Let $r_{\mathrm{FFT}}$ denote the number of retained real and imaginary rows. Conjugate redundancy and zero imaginary parts at self-conjugate frequencies mean that this row count need not equal a covariance rank. The distinction is used explicitly in the penalty initializer below.

\subsubsection{Penalty initialization and refinement}
\label{app:penalty-calibration}

For an affine residual $\mathcal R_{\mathrm{cal}}(x)=L_{\mathrm{cal}}x-b_{\mathrm{cal}}$ under $\mathcal N(\mu_{\mathrm{cal}},C_{\mathrm{cal}})$, define
\begin{equation}
\begin{aligned}
G_{\mathrm{cal}}&=W^{1/2}L_{\mathrm{cal}}C_{\mathrm{cal}}L_{\mathrm{cal}}^\top W^{1/2},\\
c_{\mathrm{cal}}&=W^{1/2}(L_{\mathrm{cal}}\mu_{\mathrm{cal}}-b_{\mathrm{cal}}).
\end{aligned}
\label{eq:app-b-calibration-gaussian}
\end{equation}
Let $\pi_\lambda^{\mathrm{cal}}$ be this Gaussian reweighted by $e^{-\lambda\|\mathcal R_{\mathrm{cal}}\|_W^2}$. Differentiating \eqref{eq:normconst} gives
\begin{equation}
\mathbb E_{\pi_\lambda^{\mathrm{cal}}}\|\mathcal R_{\mathrm{cal}}\|_W^2
=\operatorname{tr}\!\left[G_{\mathrm{cal}}(I+2\lambda G_{\mathrm{cal}})^{-1}\right]
+c_{\mathrm{cal}}^\top(I+2\lambda G_{\mathrm{cal}})^{-2}c_{\mathrm{cal}}.
\label{eq:restilt}
\end{equation}
This formula allows singular $G_{\mathrm{cal}}$. If $c_{\mathrm{cal}}=0$ and its $r>0$ positive eigenvalues all equal $1/r$, the reference expected score is one and the tilted expected score is $(1+2\lambda/r)^{-1}$. Thus a relative mean-square tolerance $0<\varepsilon_{\mathrm{tol}}<1$ gives
\begin{equation}
\lambda=\frac r2(\varepsilon_{\mathrm{tol}}^{-1}-1).
\label{eq:lambda-closed}
\end{equation}

This analytically solvable model provides an initializer for nonlinear residuals. The implemented rule substitutes the retained spectral row count $r_{\mathrm{FFT}}$ for $r$:
\begin{equation}
\lambda_{\mathrm{init}}
=\frac{r_{\mathrm{FFT}}}{2}(\varepsilon_{\mathrm{tol}}^{-1}-1),
\qquad
\varepsilon_{\mathrm{tol}}=\varepsilon_{\mathrm{rel}}^2,
\qquad 0<\varepsilon_{\mathrm{rel}}<1,
\label{eq:app-b-penalty-initializer}
\end{equation}
where $\varepsilon_{\mathrm{rel}}$ is the prescribed relative root-mean-square residual level. Because a row count is not generally the effective rank and the actual energy includes nonlinear and initial-value residuals, this is an initialization rather than an exact tolerance identity for the full target.

Let $\bar h$ estimate the mean selected score from a calibration chain at a current $\lambda>0$. When $0<\bar h<1$, fitting the response $(1+2\lambda/r_{\mathrm{eff}})^{-1}$ gives
\begin{equation}
r_{\mathrm{eff}}=\frac{2\lambda\bar h}{1-\bar h},
\qquad
\lambda'=\frac{r_{\mathrm{eff}}}{2}(\varepsilon_{\mathrm{tol}}^{-1}-1).
\label{eq:app-b-penalty-refinement}
\end{equation}
Here $r_{\mathrm{eff}}$ is a fitted response parameter, not a covariance-rank estimate. The current penalty is retained when the stated conditions fail. Finite-chain estimates and the response approximation make this a calibration against the requested tolerance, as stated in the main text, rather than a guarantee of exact moment matching. The penalty and residual scaling are finalized before the global augmentation and denoiser training.

\subsubsection{Static marginals and scalar mean calibration}
\label{app:volume}
\label{app:scalar-reference}
\label{app:scalar-calibration}

Physical weighting couples state and static fields and can shift the static marginal. For $\bar\rho(\mathrm dz,\mathrm du)=\bar\rho_z(\mathrm dz)\bar\rho(\mathrm du\mid z)$, integration over the state gives
\begin{equation}
\pi_z(\mathrm dz)\propto\bar\rho_z(\mathrm dz)\Phi(z),
\qquad
\Phi(z)=\mathbb E_{\bar\rho(\mathrm du\mid z)}[e^{-E(z,u)}].
\label{eq:phiu}
\end{equation}
The scalar adjustment in Section~\ref{sec:method:tilt} targets a prescribed expectation of the spatial mean while retaining this physical coupling.

Let $\rho_{\mathrm b}=\mathcal N(\mu_{\mathrm b},C_0)$ be the base before adjustment. Multiplication by $e^{-\vartheta\mathbf1^\top z}$ changes only its mean:
\begin{equation}
\begin{aligned}
\rho_{0,\vartheta}&=\mathcal N(\mu_{0,\vartheta},C_0),\\
(\mu_{0,\vartheta})_z&=(\mu_{\mathrm b})_z-\vartheta C_z\mathbf1,
\qquad
(\mu_{0,\vartheta})_u=(\mu_{\mathrm b})_u,
\end{aligned}
\label{eq:scalar-reference}
\end{equation}
where $\mathbf1\in\mathbb R^{d_z}$ is the all-ones vector. On the uniform grid used for this calibration, the spatial mean, prescribed value, and target calibration condition are
\begin{equation}
\psi(x)=d_z^{-1}\mathbf1^\top z,
\qquad
\psi_{\mathrm{cal}}=\mathbb E_{\rho_{\mathrm b}}\psi(X),
\qquad
\mathbb E_{\pi_{\varnothing,\vartheta}}\psi(X)=\psi_{\mathrm{cal}}.
\label{eq:moment-condition}
\end{equation}
Here $\pi_{\varnothing,\vartheta}$ is the physically weighted target with no measurement conditioning and base $\rho_{0,\vartheta}$. Thus the prescribed moment is selected before inference measurements arrive. For fixed residual construction and penalty, differentiation of this exponential family gives
\begin{equation}
\frac{\mathrm d}{\mathrm d\vartheta}
\mathbb E_{\pi_{\varnothing,\vartheta}}\psi(X)
=-d_z\operatorname{Var}_{\pi_{\varnothing,\vartheta}}(\psi(X))\le0.
\label{eq:app-b-static-mean-response}
\end{equation}
This monotonicity motivates a capped linear-fit or secant update using approximate target samples from Appendix~\ref{app:pcn}. The procedure targets one mean; it preserves the covariance of the Gaussian base but does not prescribe the full static marginal of the Gibbs target. Finite-chain moment estimates make the numerical adjustment approximate. Changes to residual scaling or $\lambda$ can alter the achieved mean, so calibration must be assessed under the final model. Subsequent measurements condition that model and may change its static-field mean. The main text uses $\rho_0$ and $\mu_0$ for the final adjusted base and mean.

\subsubsection{Absorbing linearized physics}
\label{app:augmented}

With calibration fixed, choose a global expansion point $p_0$ and define
\begin{equation}
\begin{aligned}
J_0&=\nabla\mathcal R(p_0),
\qquad b_0=J_0p_0-\mathcal R(p_0),\\
E_{\mathrm{lin}}(x)
&=\lambda\|\mathcal R(p_0)+J_0(x-p_0)\|_W^2
=\lambda\|J_0x-b_0\|_W^2.
\end{aligned}
\label{eq:app-b-global-linearization}
\end{equation}
For $\lambda>0$, set $V_0=(2\lambda)^{-1}W^{-1}$ and $S_0=J_0\bar C_{\mathcal C}J_0^\top+V_0$. Theorem~\ref{thm:quadratic} gives the augmented reference parameters:
\begin{equation}
\begin{aligned}
\tilde\mu_{\mathcal C}
&=\bar\mu_{\mathcal C}+\bar C_{\mathcal C}J_0^\top S_0^{-1}(b_0-J_0\bar\mu_{\mathcal C}),\\
\tilde C_{\mathcal C}
&=\bar C_{\mathcal C}-\bar C_{\mathcal C}J_0^\top S_0^{-1}J_0\bar C_{\mathcal C}
=(\bar C_{\mathcal C}^{-1}+2\lambda J_0^\top WJ_0)^{-1}.
\end{aligned}
\label{eq:augmented-app}
\end{equation}
The quadratic contribution is therefore handled by Gaussian conditioning, leaving the remaining energy for denoising. In particular, defining $Z_{\mathrm{lin}}=\mathbb E_{\bar\rho_{\mathcal C}}[e^{-E_{\mathrm{lin}}}]$ gives
\begin{equation}
\tilde\rho_{\mathcal C}(x)
=Z_{\mathrm{lin}}^{-1}\bar\rho_{\mathcal C}(x)e^{-E_{\mathrm{lin}}(x)},
\qquad
\mathbb E_{\tilde\rho_{\mathcal C}}[e^{-E_{\mathrm{cal}}(X)}]
=\frac{Z_{\mathcal C}}{Z_{\mathrm{lin}}}\in(0,\infty).
\label{eq:app-b-remaining-normalizer}
\end{equation}
Consequently, subtracting $E_{\mathrm{lin}}$ preserves the exact target in \eqref{eq:augmented}, even when $E_{\mathrm{cal}}$ is negative. If the residual is affine in the complete field, the linearization is exact, $E_{\mathrm{cal}}=0$, and the augmented reference already equals the target. At $\lambda=0$, both energies vanish and augmentation is omitted.

The global point is fitted without the measurement likelihood, using a few Gauss--Newton iterations for the calibrated base and physical energy. It need not equal a conditioned mean. The same $p_0$ is used for every measurement configuration; hence $E_{\mathrm{cal}}$ and the exact energy-weighted map $D$ remain fixed while their Gaussian inputs change.

\paragraph{Gaussian computations.}
\label{app:gaussian-numerics}
Matheron's rule is applied to the measurements and then to the virtual observation with operator $J_0$, value $b_0$, and covariance $V_0$. Solves with $S_0$ use Jacobian actions and iterative methods, or Cholesky factorization when the matrix can be assembled. Coordinate variances can be evaluated from a factorization or estimated from centered reference draws using
\begin{equation}
\operatorname{diag}\tilde C_{\mathcal C}
=\mathbb E\!\left[
(Z-\tilde\mu_{\mathcal C})\odot(Z-\tilde\mu_{\mathcal C})
\right],
\qquad
Z\sim\tilde\rho_{\mathcal C}.
\label{eq:app-b-gaussian-coordinate-variances}
\end{equation}
For each configuration, the same coordinate scales are used in the network inputs and output rescaling. Finite solves, variance estimates, and covariance approximations contribute numerical error separately from the exact identities above.
\section{Training the energy-determined conditional-mean diffusion denoiser}
\label{app:denoiser}

This appendix derives the conditional-mean denoiser, its energy-induced correction, and the target-sample regression used in Section~\ref{sec:method:denoiser}. It then describes the local Gaussian proposals and measurement configurations that generate training fields under the fixed calibrated model.

For a fixed measurement configuration $\mathcal C$, use the main-text shorthand $\tilde\mu=\tilde\mu_{\mathcal C}$ and $\tilde C=\tilde C_{\mathcal C}\succ0$. The augmented factorization is
\begin{equation}
\pi_{\mathcal C}(x)
=
\tilde Z_{\mathcal C}^{-1}\tilde\rho_{\mathcal C}(x)e^{-E_{\mathrm{cal}}(x)},
\qquad
\tilde Z_{\mathcal C}
=
\mathbb E_{X\sim\tilde\rho_{\mathcal C}}
\left[e^{-E_{\mathrm{cal}}(X)}\right],
\qquad
\tilde\rho_{\mathcal C}=\mathcal N(\tilde\mu,\tilde C).
\label{eq:app-c-augmented-factorization}
\end{equation}
The energy $E_{\mathrm{cal}}=E-E_{\mathrm{lin}}$ is fixed across configurations; measurements change the Gaussian reference. All denoising identities below use this same factorization as the main text. Local references used for pCN proposals are introduced separately in Appendix~\ref{app:gaussian-proposals} and do not redefine the denoiser.

\paragraph{Integrability of the denoising queries.}
Although $E_{\mathrm{cal}}$ can be negative, the Gaussian queries used by the method have finite normalizers and polynomial moments. Indeed, let $K_0=2\lambda J_0^\top WJ_0\succeq0$. Since $E\ge0$ and $\tilde C^{-1}=\bar C_{\mathcal C}^{-1}+K_0$, the unnormalized density $\mathcal N(x;m,\tau\tilde C)e^{-E_{\mathrm{cal}}(x)}$ is bounded above by $\mathcal N(x;m,\tau\tilde C)e^{E_{\mathrm{lin}}(x)}$, whose quadratic precision is
\begin{equation}
\tau^{-1}\tilde C^{-1}-K_0
=
\tau^{-1}\bar C_{\mathcal C}^{-1}
+(\tau^{-1}-1)K_0
\succ0,
\qquad 0<\tau\le1.
\label{eq:app-c-query-precision}
\end{equation}
The upper bound is therefore an integrable unnormalized Gaussian for every $m\in\mathbb R^d$. Positivity follows from finiteness of the residual energy almost everywhere. Thus the required moments hold for the exact query family $\Sigma=\tau\tilde C$; other Gaussian queries require their own integrability conditions. Differentiation and integration by parts use the additional regularity specified below.

\subsection{Reference preservation}
\label{app:reference-preservation}

\begin{lemma}[Reference-preserving Gaussian perturbation]
\label{lem:forward}
Let $X\sim\tilde\rho_{\mathcal C}$ and $\xi_{\mathrm G}\sim\mathcal N(0,I_d)$ be independent. For $a\in[0,1]$ and $\tau=1-a^2$, define
\begin{equation}
X_a=\tilde\mu+a(X-\tilde\mu)
+\sqrt{\tau}\,\tilde C^{1/2}\xi_{\mathrm G}.
\label{eq:app-c-reference-perturbation}
\end{equation}
Then $X_a\sim\tilde\rho_{\mathcal C}$. For $0\le a<1$, the reference conditional density is
\begin{equation}
\begin{aligned}
q_a(x\mid x_a)&=\mathcal N(x;m_a(x_a),\Sigma_a),\\
m_a(x_a)&=\tilde\mu+a(x_a-\tilde\mu),
\qquad \Sigma_a=\tau\tilde C.
\end{aligned}
\label{eq:reference_backward}
\end{equation}
At $a=0$, $X_a$ is independent of $X$; at $a=1$, $X_a=X$ and the conditional distribution is a point mass.
\end{lemma}

\begin{proof}
The pair $(X,X_a)$ is jointly Gaussian, with
\begin{equation}
\mathbb E[X_a]=\tilde\mu,
\qquad
\operatorname{Cov}(X_a)=a^2\tilde C+\tau\tilde C=\tilde C,
\qquad
\operatorname{Cov}(X,X_a)=a\tilde C.
\label{eq:app-c-reference-moments}
\end{equation}
Gaussian conditioning gives \eqref{eq:reference_backward}; the endpoints follow directly from \eqref{eq:app-c-reference-perturbation}.
\end{proof}

The reference conditional mean $m_a$ is the analytic Gaussian denoising estimate. When $X$ instead follows $\pi_{\mathcal C}$, the same perturbation connects the Gaussian reference at $a=0$ to the Gibbs target at $a=1$. The next result identifies the correction to $m_a$ required by this target.

\subsection{Proof of Theorem~\ref{thm:denoiser}}
\label{app:denoiser-proof}

For a Gaussian query with the required integrability, define its energy normalizer and tilted law by
\begin{equation}
\begin{aligned}
\Phi(m,\Sigma)
&=\mathbb E_{X\sim\mathcal N(m,\Sigma)}
\left[e^{-E_{\mathrm{cal}}(X)}\right],\\
P_{m,\Sigma}(\mathrm dx)
&=\Phi(m,\Sigma)^{-1}
\mathcal N(x;m,\Sigma)e^{-E_{\mathrm{cal}}(x)}\,\mathrm dx.
\end{aligned}
\label{eq:app-c-weighted-query}
\end{equation}
Its mean is $D(m,\Sigma)$ from \eqref{eq:denoiser}. For later use, write the mean correction and its normalization as
\begin{equation}
\Delta(m,\Sigma)=D(m,\Sigma)-m,
\qquad
\widetilde\Delta(m,\Sigma;\tau)=\frac{\Delta(m,\Sigma)}{\tau},
\qquad \Sigma=\tau\tilde C,\quad\tau>0.
\label{eq:app-c-exact-corrections}
\end{equation}
The multiplier $\tau$ remains explicit because $\tilde C$ varies with the measurement configuration.

\begin{proof}[Proof of Theorem~\ref{thm:denoiser}]
Fix $\mathcal C$ and $0\le a<1$. Let $k_a(x_a\mid x)$ denote the Gaussian perturbation density, and let $p_a$ be its marginal when $X\sim\pi_{\mathcal C}$. Bayes' rule and Lemma~\ref{lem:forward} give
\begin{equation}
\begin{aligned}
p(x\mid X_a=x_a)
&\propto\tilde\rho_{\mathcal C}(x)k_a(x_a\mid x)
 e^{-E_{\mathrm{cal}}(x)}\\
&\propto q_a(x\mid x_a)e^{-E_{\mathrm{cal}}(x)}.
\end{aligned}
\label{eq:app-c-target-conditional}
\end{equation}
The target conditional law is therefore $P_{m_a(x_a),\Sigma_a}$. Its mean is $D(m_a(x_a),\Sigma_a)$, proving \eqref{eq:conditional-denoiser}.
\end{proof}

\paragraph{Energy-induced mean correction.}
The conditional-mean identity requires no derivatives of the energy. To obtain the gradient representation, additionally assume that $E_{\mathrm{cal}}$ is continuously differentiable, its gradient is integrable under the tilted query, and the density derivatives are integrable with vanishing boundary terms. Integrating the gradient of the unnormalized density in \eqref{eq:app-c-weighted-query} gives
\begin{equation}
0
=
\mathbb E_{P_{m,\Sigma}}
\left[-\Sigma^{-1}(X-m)-\nabla E_{\mathrm{cal}}(X)\right].
\label{eq:app-c-integration-by-parts}
\end{equation}
Consequently,
\begin{equation}
\Delta(m,\Sigma)
=-\Sigma\,
\mathbb E_{P_{m,\Sigma}}
\left[\nabla E_{\mathrm{cal}}(X)\right].
\label{eq:denoiser_stein_identity}
\end{equation}
For $(m,\Sigma)=(m_a,\Sigma_a)$, this is exactly \eqref{eq:energy-induced-correction}. It expresses the denoising correction as a covariance-scaled conditional average of the remaining energy gradient. The PDE enters this identity through
\begin{equation}
\begin{aligned}
\nabla E_{\mathrm{cal}}(x)
&=\nabla E(x)-\nabla E_{\mathrm{lin}}(x)\\
&=2\lambda J(x)^\top W\mathcal R(x)
-2\lambda J_0^\top W
\left[\mathcal R(p_0)+J_0(x-p_0)\right],
\qquad J(x)=\nabla\mathcal R(x).
\end{aligned}
\label{eq:app-c-energy-gradient}
\end{equation}
The quadratic contribution is already represented analytically in $\tilde\rho_{\mathcal C}$; the denoiser learns the additional mean correction induced by the remainder.

\paragraph{Score identity.}
Reference preservation also factorizes the perturbed target density. When differentiation with respect to $m$ is justified by local domination,
\begin{equation}
\begin{aligned}
p_a(x_a)
&=\frac{\tilde\rho_{\mathcal C}(x_a)}{\tilde Z_{\mathcal C}}
\Phi(m_a(x_a),\Sigma_a),\\
\nabla_m\log\Phi(m,\Sigma)
&=\Sigma^{-1}\Delta(m,\Sigma).
\end{aligned}
\label{eq:app-c-perturbed-density}
\end{equation}
Using $\nabla_{x_a}m_a=aI_d$ and $\Sigma_a=\tau\tilde C$ yields
\begin{equation}
\begin{aligned}
\nabla_{x_a}\log p_a(x_a)
&=-\tilde C^{-1}(x_a-\tilde\mu)
+a\Sigma_a^{-1}\Delta(m_a,\Sigma_a)\\
&=-\tilde C^{-1}(x_a-\tilde\mu)
+a\tilde C^{-1}\widetilde\Delta(m_a,\Sigma_a;\tau).
\end{aligned}
\label{eq:score}
\end{equation}
Here $m_a=m_a(x_a)$. The first term is the Gaussian reference score, and the second is the energy-induced correction. This decomposition supplies the cancellation in \eqref{eq:exact-flow-velocity}.

\paragraph{Small-noise limit.}
Suppose $E_{\mathrm{cal}}$ is continuously differentiable near $m$, the tilted queries $P_{m,\tau\tilde C}$ concentrate at $m$, and their energy gradients are uniformly integrable as $\tau\downarrow0$. Then the gradient expectation in \eqref{eq:denoiser_stein_identity} converges to $\nabla E_{\mathrm{cal}}(m)$, giving
\begin{equation}
\begin{aligned}
\Delta(m,\tau\tilde C)
&=-\tau\tilde C\nabla E_{\mathrm{cal}}(m)+o(\tau),\\
\widetilde\Delta(m,\tau\tilde C;\tau)
&\longrightarrow-\tilde C\nabla E_{\mathrm{cal}}(m).
\end{aligned}
\label{eq:denoiser_small_noise}
\end{equation}
For the moving mean $m_a(x)=\tilde\mu+a(x-\tilde\mu)$ used in sampling, the same conclusion at $x$ follows when $P_{m_a(x),(1-a^2)\tilde C}$ concentrates at $x$ and its energy gradients are uniformly integrable as $a\uparrow1$. Under these conditions, \eqref{eq:score} approaches the Gibbs target score. Appendix~\ref{app:probability-flow} uses this limit to interpret the terminal sampling velocity.

\subsection{Target draws as unbiased regression labels}
\label{app:unbiased}

\begin{lemma}[Single-draw labels]
\label{lem:unbiased}
Fix a measurement configuration $\mathcal C$ and $0\le a<1$. Let $X\sim\pi_{\mathcal C}$ and construct $X_a$ using \eqref{eq:forward}. With $s_a=\sqrt{\operatorname{diag}\Sigma_a}$,
\begin{equation}
\begin{aligned}
\mathbb E\!\left[
\frac{X-m_a(X_a)}{\tau}\,\middle|\,X_a
\right]
&=\widetilde\Delta(m_a(X_a),\Sigma_a;\tau),\\
\mathbb E\!\left[
\frac{X-m_a(X_a)}{s_a}\,\middle|\,X_a
\right]
&=\frac{\Delta(m_a(X_a),\Sigma_a)}{s_a},
\end{aligned}
\label{eq:unbiased}
\end{equation}
where division by $s_a$ is elementwise.
\end{lemma}

\begin{proof}
Subtract $m_a(X_a)$ from the conditional-mean identity in Theorem~\ref{thm:denoiser} and divide by the fixed $\tau$ or $s_a$. For random configurations and signal coefficients, both are included in the conditioning.
\end{proof}

For square-integrable labels, squared-error regression learns their conditional mean. The identities in \eqref{eq:unbiased} therefore turn target draws into supervised labels for the standardized energy-induced correction, without evaluating either the energy-weighted Gaussian expectation or the conditional energy-gradient expectation. The network output is rescaled and normalized in \eqref{eq:denoiser-displacement} for use as a sampling velocity. Appendix~\ref{app:covariance-encoding} characterizes this regression when one network is trained across configurations with compressed covariance inputs.

\paragraph{Why standardize the labels?}
Let $T_a=\operatorname{diag}(s_a)$. The standardized label has conditional covariance
\begin{equation}
\operatorname{Cov}\!\left(
\frac{X-m_a}{s_a}\,\middle|\,X_a
\right)
=
T_a^{-1}\operatorname{Cov}_{P_{m_a,\Sigma_a}}(X)T_a^{-1}.
\label{eq:app-c-standardized-label-covariance}
\end{equation}
This scaling accounts for coordinate-dependent uncertainty. As an analytic case, suppose $E_{\mathrm{cal}}(x)=\tfrac12x^\top Hx+b^\top x+c$, with fixed $H\succeq0$. Completing the square gives $\Sigma_a^\star=(\Sigma_a^{-1}+H)^{-1}$, and hence
\begin{equation}
\begin{aligned}
\operatorname{Cov}\!\left(
\frac{X-m_a}{\tau}\,\middle|\,X_a
\right)
&=\frac{\Sigma_a^\star}{\tau^2}
=\frac{\tilde C}{\tau}+O(1)
\quad(\tau\downarrow0),\\
\operatorname{Cov}\!\left(
\frac{X-m_a}{s_a}\,\middle|\,X_a
\right)
&=T_a^{-1}\Sigma_a^\star T_a^{-1}
\preceq T_a^{-1}\Sigma_aT_a^{-1}.
\end{aligned}
\label{eq:label_variance_affine}
\end{equation}
The upper bound is the correlation matrix of $\tilde C$, so each standardized coordinate variance is at most one in this quadratic case. For a residual affine in the complete field, exact augmentation gives $E_{\mathrm{cal}}=0$ and attains the bound. For nonlinear remaining energies, a conditional covariance of $\tau\tilde C+o(\tau)$ likewise gives a finite standardized small-noise limit, subject to the required moment convergence. Standardization controls coordinate scales; it does not decorrelate the labels.

The identities concern exact target marginals. Correlation among stationary chain samples affects effective sample size, whereas incomplete equilibration and selection of retained pools can also change the regression target. The pCN procedure below therefore supplies approximate training draws, consistently with Section~\ref{sec:method:denoiser}.

\subsection{Local Gaussian references for training}
\label{app:gaussian-proposals}

A local quadratic approximation provides a Gaussian proposal reference while the exact residual energy remains in the target density. For a fixed local expansion point $p\in\mathbb R^d$, define
\begin{equation}
J_p=\nabla\mathcal R(p),
\qquad
b_J=J_pp-\mathcal R(p),
\qquad
E_{\mathrm{lin}}^{(p)}(x)=\lambda\|J_px-b_J\|_W^2.
\label{eq:app-c-local-linearization}
\end{equation}
For $\lambda>0$ and $\Sigma\succ0$, Theorem~\ref{thm:quadratic} gives the Gaussian tilt $Q_J=\mathcal N(m^\star,\Sigma^\star)$ of $\mathcal N(m,\Sigma)$ by this energy:
\begin{equation}
\begin{aligned}
S_J&=J_p\Sigma J_p^\top+(2\lambda)^{-1}W^{-1},\\
m^\star&=m+\Sigma J_p^\top S_J^{-1}(b_J-J_pm),\\
\Sigma^\star&=(\Sigma^{-1}+2\lambda J_p^\top WJ_p)^{-1}.
\end{aligned}
\label{eq:gnproposal}
\end{equation}
Training chains use $(m,\Sigma)=(\bar\mu_{\mathcal C},\bar C_{\mathcal C})$, so $Q_J\propto\bar\rho_{\mathcal C}e^{-E_{\mathrm{lin}}^{(p)}}$. The local point $p$ can depend on the training configuration; the global point $p_0$ and the denoising energy $E_{\mathrm{cal}}$ remain fixed. Thus $Q_J$ defines chain proposals, whereas $\tilde\rho_{\mathcal C}$ defines the Gaussian perturbation used for training and sampling. The two quadratic approximations are alternative reference factorizations and are not accumulated.

Matheron's rule generates a draw from $Q_J$ using independent $X_{\mathrm p}\sim\mathcal N(m,\Sigma)$ and $\varepsilon_{\mathrm p}\sim\mathcal N(0,(2\lambda)^{-1}W^{-1})$:
\begin{equation}
X^\star=X_{\mathrm p}+\Sigma J_p^\top S_J^{-1}
(b_J-J_pX_{\mathrm p}-\varepsilon_{\mathrm p}),
\qquad X^\star\sim Q_J.
\label{eq:proposal_matheron}
\end{equation}
At $\lambda=0$, the quadratic update is omitted and $Q_J=\mathcal N(m,\Sigma)$. Covariance products, Jacobian actions, and linear solves follow Appendix~\ref{app:gaussian-numerics}, with $J_p$ replacing $J_0$. The exact proposal formulas assume consistent Gaussian draws and solves; numerical approximations in these operations contribute to training-chain error.

\subsection{pCN sampling of the Gibbs target}
\label{app:pcn}

\paragraph{Density-ratio correction.}
Dividing the target density by the local Gaussian reference cancels the measurement-conditioned Gaussian factor:
\begin{equation}
\frac{\pi_{\mathcal C}(x)}{Q_J(x)}
\propto e^{-\delta_{\mathrm{def}}(x)},
\qquad
\delta_{\mathrm{def}}(x)=E(x)-E_{\mathrm{lin}}^{(p)}(x).
\label{eq:app-c-proposal-defect}
\end{equation}
In terms of the globally augmented factorization, the same defect is
\begin{equation}
\delta_{\mathrm{def}}
=E_{\mathrm{cal}}-
\left(E_{\mathrm{lin}}^{(p)}-E_{\mathrm{lin}}\right).
\label{eq:app-c-defect-equivalence}
\end{equation}
The correction therefore retains the full selected residual energy defining $\pi_{\mathcal C}$, even though the proposal uses a local quadratic approximation. Evaluating the energy requires residual evaluation on a candidate complete field, rather than a forward PDE solution for each candidate parameter field.

\paragraph{The pCN transition.}
Write $Q_J=\mathcal N(m_Q,C_Q)$, with $(m_Q,C_Q)=(m^\star,\Sigma^\star)$ from \eqref{eq:gnproposal}. For a fixed step size $b_{\mathrm{pcn}}\in(0,1]$ and an independent innovation $\zeta\sim\mathcal N(0,C_Q)$, propose and accept according to
\begin{equation}
\begin{aligned}
x'&=m_Q+\sqrt{1-b_{\mathrm{pcn}}^2}\,(x-m_Q)+b_{\mathrm{pcn}}\zeta,\\
\mathbb P(\mathrm{accept})
&=\min\left\{1,
\exp[-\delta_{\mathrm{def}}(x')+\delta_{\mathrm{def}}(x)]
\right\}.
\end{aligned}
\label{eq:pcn}
\end{equation}
An innovation is obtained by centering a draw from $Q_J$, using \eqref{eq:proposal_matheron} when $\lambda>0$. With exact Gaussian draws and fixed $p$ and $b_{\mathrm{pcn}}$, the proposal is reversible with respect to $Q_J$. Combining that reversibility with \eqref{eq:app-c-proposal-defect} yields detailed balance for $\pi_{\mathcal C}$~\citep{cotter2013mcmc}. Rejected proposals retain the current state. Under the additional ergodicity conditions for the chain, long-run averages converge to target expectations.

\paragraph{Adaptation and retained training fields.}
Gauss--Newton iterations initialize the local point $p$. During burn-in, the implementation adapts the step size and may recenter $p$ on the running walker mean when this decreases the empirical defect variance. Current-state defects are recomputed after each reference change. Both $p$ and the step size are fixed before retaining fields. An initial segment is discarded, and each walker's current state is recorded at fixed intervals, including repeated states after rejections.

This separation allows adaptation to improve the proposal while retaining a fixed target-preserving kernel for subsequent sampling. Finite-chain output is nevertheless approximate: freezing the kernel does not establish equilibration, and thinning does not guarantee independence. For affine residuals, the defect vanishes and every proposal is accepted; consecutive states become independent only at $b_{\mathrm{pcn}}=1$.

The training pipeline also screens finite pools using residual and chain summaries and can redraw configurations. These retention decisions can change the effective training distribution and do not establish equilibration or the moment assumptions required for regression. The exact label identities in Appendix~\ref{app:unbiased} describe the target population; finite-chain, numerical, and retention effects contribute to its practical approximation.

\subsection{Covariance encoding and the regression target}
\label{app:covariance-encoding}

With $E_{\mathrm{cal}}$ fixed, the exact map $D(m,\Sigma)$ is shared across measurement configurations. The full covariance still enters Gaussian conditioning, reference sampling, and perturbation generation. For the network input, PhysDEM uses the coordinate scales $s=\sqrt{\operatorname{diag}\Sigma}$ together with $m$, $a$, and the fixed grid coordinates. This provides a field-sized representation whose dimension does not grow with the number of measurements, while leaving the interpretation of its regression target explicit.

For random training configurations and signal coefficients, let
\begin{equation}
L_{\mathrm{reg}}=\frac{X-m_a}{s_a},
\qquad
F=(m_a,s_a,a),
\qquad
T_{\mathrm{reg}}=\frac{\Delta(m_a,\Sigma_a)}{s_a}.
\label{eq:app-c-encoded-variables}
\end{equation}
Assume exact conditional target draws, the same fixed energy, and square-integrable labels. Lemma~\ref{lem:unbiased} gives $T_{\mathrm{reg}}$ as the conditional label mean given the full configuration, $a$, and $X_a$. Since $F$ is determined by those quantities, the tower property gives the population optimum for the encoded inputs:
\begin{equation}
n^\star(F)
=\mathbb E[L_{\mathrm{reg}}\mid F]
=\mathbb E[T_{\mathrm{reg}}\mid F].
\label{eq:app-c-encoded-regression}
\end{equation}
Thus training learns the best squared-error predictor of the standardized correction from the supplied inputs. More explicitly, for any square-integrable predictor $n(F)$, conditional orthogonality gives
\begin{equation}
\begin{aligned}
\mathbb E\|L_{\mathrm{reg}}-n(F)\|^2
={}&\mathbb E\|L_{\mathrm{reg}}-T_{\mathrm{reg}}\|^2\\
&+\mathbb E\|T_{\mathrm{reg}}-n^\star(F)\|^2
+\mathbb E\|n^\star(F)-n(F)\|^2.
\end{aligned}
\label{eq:app-c-regression-decomposition}
\end{equation}
The terms separate label variability, covariance-encoding error, and predictor approximation error under the training law.

\paragraph{When the encoding is sufficient.}
The encoding is sufficient for the correction whenever $T_{\mathrm{reg}}$ is a measurable function of $F$. A sufficient condition is that $\Sigma_a$ can be recovered from $(m_a,s_a,a)$ on the configuration family of interest. For example, if the queries share a known correlation matrix $K$, then
\begin{equation}
\Sigma_a=\operatorname{diag}(s_a)\,K\,\operatorname{diag}(s_a),
\label{eq:app-c-encoding-sufficiency}
\end{equation}
and the coordinate scales determine the full covariance. Diagonal covariance is the special case $K=I_d$. In these cases the middle term in \eqref{eq:app-c-regression-decomposition} vanishes, and the population predictor matches the exact standardized correction configuration by configuration.

For more general configurations, the diagonal encoding approximates the dependence on covariance correlations. The conditional projection in \eqref{eq:app-c-encoded-regression} characterizes its population prediction, and \eqref{eq:app-c-regression-decomposition} identifies any remaining encoding error. Reuse without retraining means that new measurements update the Gaussian inputs of a shared network; its accuracy depends on the adequacy of this encoding and generalization to those inputs, as stated in Sections~\ref{sec:method:denoiser} and~\ref{sec:method:sample}.

\subsection{Training configurations}
\label{app:configs}

The configuration generator first selects a pseudo-field from an approximate chain pool for the calibrated target without measurement conditioning:
\begin{equation}
\pi_{\varnothing}(x)
=Z_{\varnothing}^{-1}\rho_0(x)e^{-E(x)}.
\label{eq:app-c-unconditioned-target}
\end{equation}
The implementation samples a measurement-group count from a fixed set including zero, selects spatial grid locations, and retains a random fraction of time levels. Measurements are simulated from the pseudo-field through \eqref{eq:obs}, with the specified noise covariances; measurements of the static field $z$ may also be included. The resulting $\mathcal C$ contains both the observation setup and the simulated values.

For each configuration, the measurement-conditioned reference is constructed and the pCN procedure in Appendix~\ref{app:pcn} supplies a pool of approximate conditional target samples. Retained fields are perturbed using the globally augmented reference $\tilde\rho_{\mathcal C}$. Independent draws of $a$ and $\xi_{\mathrm G}$ in \eqref{eq:forward} produce the regression labels in \eqref{eq:loss}, and the pool is refreshed after a prescribed number of training steps. All complete training fields and measurement values are generated under the fixed physical model. Training therefore requires neither a preassembled full-field dataset from external simulators or real-world measurements nor the measurements later supplied at inference.

\section{Sampling with the physics-displacement probability flow}
\label{app:sampler}

This appendix derives the sampler in Section~\ref{sec:method:sample}, establishes how measurement updates enter the fixed model, and gives the expected coverage result used to interpret generated field ensembles.

\subsection{Probability flow and numerical sampling}
\label{app:probability-flow}

Fix a measurement configuration $\mathcal C$ and use the notation of the main text: $\tilde\rho_{\mathcal C}=\mathcal N(\tilde\mu,\tilde C)$, $\tilde C\succ0$, and $\pi_{\mathcal C}\propto\tilde\rho_{\mathcal C}e^{-E_{\mathrm{cal}}}$. The calibrated model and global expansion point $p_0$ are fixed. The denoiser $D$ in \eqref{eq:denoiser} is defined by this same remaining energy $E_{\mathrm{cal}}$.

\paragraph{Reference-preserving diffusion and Gaussian cancellation.}
Introduce an auxiliary time $v\in[0,v_{\mathrm{end}})$ and a continuously differentiable schedule $a_v$ with $a_0=1$, $a_v>0$, $\mathrm da_v/\mathrm dv<0$ in the interior, and $a_v\to0$ as $v\to v_{\mathrm{end}}$. Let $B_v$ be standard $d$-dimensional Brownian motion. Define
\begin{equation}
\beta_v=-2\frac{\mathrm d}{\mathrm dv}\log a_v>0,
\qquad
\mathrm dZ_v
=
-\frac{\beta_v}{2}(Z_v-\tilde\mu)\,\mathrm dv
+\sqrt{\beta_v}\,\tilde C^{1/2}\mathrm dB_v,
\qquad
Z_0\sim\pi_{\mathcal C}.
\label{eq:app-d-forward-sde}
\end{equation}
For locally integrable $\beta_v$, the solution has the Gaussian-perturbation marginals
\begin{equation}
a_v=\exp\!\left(-\frac12\int_0^v\beta_r\,\mathrm dr\right),
\qquad
Z_v\overset{\mathrm d}{=}
\tilde\mu+a_v(X-\tilde\mu)
+\sqrt{1-a_v^2}\,\tilde C^{1/2}\xi_{\mathrm G},
\label{eq:app-d-forward-marginals}
\end{equation}
where $X\sim\pi_{\mathcal C}$ and $\xi_{\mathrm G}\sim\mathcal N(0,I_d)$ are independent. The limit $a_v\to0$ corresponds to an infinite integrated noise rate; $v_{\mathrm{end}}$ may be infinite. If initialized from $\tilde\rho_{\mathcal C}$, this diffusion preserves that reference, as established in Lemma~\ref{lem:forward}.

Let $p_a$ be the density of the perturbation in \eqref{eq:forward}, and write $m(a,x)=\tilde\mu+a(x-\tilde\mu)$ and $\Sigma_a=(1-a^2)\tilde C$. The score identity \eqref{eq:score} separates the Gaussian reference score from the energy-induced correction:
\begin{equation}
\nabla_x\log p_a(x)
=
-\tilde C^{-1}(x-\tilde\mu)
+
\frac{a}{1-a^2}\tilde C^{-1}
\bigl[D(m(a,x),\Sigma_a)-m(a,x)\bigr].
\label{eq:app-d-score-decomposition}
\end{equation}
The probability-flow ODE associated with \eqref{eq:app-d-forward-sde} has the same marginal densities under the usual regularity and uniqueness conditions~\citep{song2021score}. Substituting \eqref{eq:app-d-score-decomposition} gives
\begin{equation}
\begin{aligned}
\frac{\mathrm dX}{\mathrm dv}
&=
-\frac{\beta_v}{2}
\left[(X-\tilde\mu)+\tilde C\nabla_x\log p_{a_v}(X)\right]
\\
&=
-\frac{\beta_va_v}{2}
\frac{D(m(a_v,X),\Sigma_{a_v})-m(a_v,X)}{1-a_v^2}.
\end{aligned}
\label{eq:app-d-flow-cancellation}
\end{equation}
The linear drift cancels the Gaussian score contribution exactly. Since $\mathrm da_v/\mathrm dv=-\beta_va_v/2$, changing to the signal coefficient yields
\begin{equation}
\begin{aligned}
v_a(x)
&=
\frac{x-\tilde\mu+\tilde C\nabla_x\log p_a(x)}{a}
=
\frac{\tilde C}{a}\nabla_x\log\!\left[
\frac{p_a(x)}{\tilde\rho_{\mathcal C}(x)}
\right]
\\
&=
\frac{D(m(a,x),\Sigma_a)-m(a,x)}{1-a^2}
=
-\tilde C\,
\mathbb E\!\left[
\nabla E_{\mathrm{cal}}(X)\mid X_a=x
\right],
\qquad 0<a<1.
\end{aligned}
\label{eq:app-d-exact-velocity}
\end{equation}
The last equality uses the integration-by-parts conditions in Appendix~\ref{app:denoiser-proof}. The mean-correction formula itself follows from the conditional-mean and score identities. Thus the quantity learned by energy matching, after restoring its coordinate scale and dividing by $1-a^2$, is precisely the velocity required for sampling. Gaussian conditioning still enters through $\tilde\mu$, $\tilde C$, and initialization; the cancellation isolates the correction that the network must supply.

\paragraph{Exact transport and endpoints.}
Sampling follows increasing $a$ from the Gaussian reference at $a=0$. For every $0<b<1$, assume the exact velocity extends continuously to $a=0$, generates a unique non-explosive flow on $[0,b]$, and its continuity equation has a unique probability-law solution initialized at $\tilde\rho_{\mathcal C}$. The flow then has marginals $p_a$, because these marginals satisfy the same continuity equation and initial condition. As $a\uparrow1$, \eqref{eq:forward} converges to the target field, so the flow distributions converge to $\pi_{\mathcal C}$. If a terminal flow map exists, it transports $\tilde\rho_{\mathcal C}$ to $\pi_{\mathcal C}$.

At $a=0$, the exact mean-correction velocity is constant. At the other endpoint, the moving-query concentration and gradient-integrability conditions in Appendix~\ref{app:denoiser-proof} give
\begin{equation}
\begin{aligned}
v_0(x)
&=
D(\tilde\mu,\tilde C)-\tilde\mu
=
\mathbb E_{\pi_{\mathcal C}}[X]-\tilde\mu,
\\
\lim_{a\uparrow1}v_a(x)
&=
-\tilde C\,\nabla E_{\mathrm{cal}}(x).
\end{aligned}
\label{eq:app-d-endpoint-velocities}
\end{equation}
These expressions describe the exact flow. The practical sampler below evaluates its learned approximation on a truncated interval.

The affine case makes the role of reference augmentation particularly transparent. If the full residual is affine, $\mathcal R(x)=Lx-b$, its linearization at any $p_0$ is exact. Hence
\begin{equation}
E_{\mathrm{lin}}=E,
\qquad
E_{\mathrm{cal}}=0,
\qquad
\tilde\rho_{\mathcal C}=\pi_{\mathcal C},
\qquad
v_a(x)=0.
\label{eq:app-d-affine-exactness}
\end{equation}
In this case, Gaussian conditioning already produces target draws. For nonlinear residuals, the learned flow supplies the remaining correction while preserving the target decomposition.

\paragraph{Learned velocity and numerical integration.}
The implementation uses the saved exponential moving average of the denoiser parameters, denoted by $\theta$, and draws $x_0$ from the augmented reference using the Gaussian operations in Appendix~\ref{app:gaussian-numerics}. It evaluates
\begin{equation}
v_{\theta,a}(x)
=
\widetilde\Delta_\theta(m(a,x),s(a),a)
=
\frac{s(a)\odot n_\theta(m(a,x),s(a),a)}{1-a^2},
\qquad
s(a)=\sqrt{(1-a^2)\operatorname{diag}\tilde C}.
\label{eq:app-d-learned-velocity}
\end{equation}
For an integer $N\ge1$, set $\delta a=a_{\max}/N$ and $a_j=j\,\delta a$. Second-order Heun steps update the complete field for $j=0,\ldots,N-1$:
\begin{equation}
\begin{aligned}
k_1&=v_{\theta,a_j}(x_j),
&
x_{\mathrm{pred}}&=x_j+\delta a\,k_1,
\\
k_2&=v_{\theta,a_{j+1}}(x_{\mathrm{pred}}),
&
x_{j+1}&=x_j+\frac{\delta a}{2}(k_1+k_2).
\end{aligned}
\label{eq:app-d-heun}
\end{equation}
The sampler returns $x_N$ at $a_{\max}<1$, without a terminal correction. Its velocity evaluations span $[0,a_{\max}]$; if $a_{\min}>0$ in training, the initial portion relies on extrapolation below that range. Independent initial draws produce independent outputs conditional on the trained network and any shared Gaussian preprocessing.

\paragraph{Approximation and endpoint truncation.}
The exact intermediate law approaches the target quantitatively. Let $W_2$ denote the quadratic Wasserstein distance and assume $\mathbb E_{\pi_{\mathcal C}}\|X\|^2<\infty$, as holds for the Gaussian-reweighted target here. Coupling $X_a$ and $X$ through \eqref{eq:forward} gives
\begin{equation}
W_2^2(p_a,\pi_{\mathcal C})
\le
(1-a)^2\mathbb E_{\pi_{\mathcal C}}\|X-\tilde\mu\|^2
+(1-a^2)\operatorname{tr}\tilde C.
\label{eq:app-d-truncation-bound}
\end{equation}
The cross term vanishes because the Gaussian noise is independent and centered. This bound isolates endpoint truncation for the exact flow.

The normalization in \eqref{eq:app-d-learned-velocity} also specifies how regression errors enter the velocity. Let $e_{\theta,i}(a,x)$ be the error in coordinate $i$ of the standardized mean correction, using exact Gaussian inputs. Then, for any fixed $a_{\max}<1$,
\begin{equation}
\begin{aligned}
v_{\theta,a,i}(x)-v_{a,i}(x)
&=
\frac{\sqrt{\tilde C_{ii}}}{\sqrt{1-a^2}}\,
e_{\theta,i}(a,x),
\\
\|v_{\theta,a}(x)-v_a(x)\|^2
&\le
\frac{\max_i\tilde C_{ii}}{1-a_{\max}^2}
\|e_\theta(a,x)\|^2,
\qquad 0\le a\le a_{\max}.
\end{aligned}
\label{eq:app-d-velocity-error}
\end{equation}
Thus a fixed truncated interval provides a finite conversion between standardized correction error and velocity error. Heun's second-order convergence holds when the implemented velocity is sufficiently regular on that interval. Training-chain approximation, covariance encoding, Gaussian numerical operations, the learned correction, integration, and truncation together determine the implemented sampling law. The coverage analysis in Appendix~\ref{app:expected-coverage} accounts for their combined distributional effect through the sampler discrepancy.

\subsection{Conditioning a fixed model}
\label{app:measurement-update}

Let $\mathcal C_0$ contain the measurements already incorporated, possibly none, and write $\tilde\rho_{\mathcal C_0}=\mathcal N(\tilde\mu,\tilde C)$. Consider an additional block $y_{\mathrm{new}}\in\mathbb R^{n_{\mathrm{new}}}$ with operator $A\in\mathbb R^{n_{\mathrm{new}}\times d}$ and noise covariance $\Gamma\succ0$, as in \eqref{eq:obs}. Its noise is independent of the field and previously incorporated measurement noise. With the calibrated model and $p_0$ fixed, the configuration becomes $\mathcal C=\mathcal C_0\cup\{(A,y_{\mathrm{new}},\Gamma)\}$, and Gaussian conditioning gives
\begin{equation}
\begin{aligned}
S&=A\tilde C A^\top+\Gamma,
\\
\tilde\mu'
&=\tilde\mu+\tilde C A^\top S^{-1}
(y_{\mathrm{new}}-A\tilde\mu),
\\
\tilde C'
&=\tilde C-\tilde C A^\top S^{-1}A\tilde C.
\end{aligned}
\label{eq:refupdate}
\end{equation}
Here $S$ is the predictive measurement covariance under the reference, and $\tilde\rho_{\mathcal C}=\mathcal N(\tilde\mu',\tilde C')$. Multiplication by the new likelihood yields
\begin{equation}
\pi_{\mathcal C}(x)
=
\pi_{\mathcal C_0}(x\mid y_{\mathrm{new}})
\propto
\mathcal N(x;\tilde\mu',\tilde C')e^{-E_{\mathrm{cal}}(x)}.
\label{eq:cond}
\end{equation}
Conditioning the augmented reference directly is equivalent to conditioning the base on all measurements and applying the same fixed augmentation: likelihood multiplication and multiplication by $e^{-E_{\mathrm{lin}}}$ commute. Each observation is incorporated once. The remaining energy $E_{\mathrm{cal}}$ and the exact denoiser map $D$ therefore remain unchanged.

The same construction accommodates different measurement layouts, values, and counts by rebuilding the Gaussian reference for the desired configuration. The trained network is reused with the resulting Gaussian inputs, with its approximation described in Appendix~\ref{app:covariance-encoding}. Reference updates use the covariance products and measurement-space solves in Appendix~\ref{app:measurement-conditioning}. This separates the fixed training cost from the conditioning and learned-flow computations required for each inference query.

\subsection{Field variability under additional measurements}
\label{app:field-variability}

The reference update in \eqref{eq:refupdate} gives an explicit covariance reduction:
\begin{equation}
\tilde C-\tilde C'
=
\tilde C A^\top S^{-1}A\tilde C
\succeq0,
\qquad
w^\top(\tilde C-\tilde C')w
=
\|S^{-1/2}A\tilde Cw\|^2
\quad\text{for }w\in\mathbb R^d.
\label{eq:app-d-reference-contraction}
\end{equation}
Thus measurements reduce reference uncertainty in directions correlated with the observed quantities, including field values at unmeasured locations. A direction is unchanged precisely when $A\tilde Cw=0$.

When the residual is affine, the exact target is Gaussian as well. Writing $V_{\mathcal C_0}$ and $V_{\mathcal C}$ for the target covariances before and after the new measurements,
\begin{equation}
V_{\mathcal C_0}
=
\left(\bar C_{\mathcal C_0}^{-1}+2\lambda L^\top WL\right)^{-1},
\qquad
V_{\mathcal C}
=
\left(V_{\mathcal C_0}^{-1}+A^\top\Gamma^{-1}A\right)^{-1}
\preceq V_{\mathcal C_0}.
\label{eq:app-d-affine-contraction}
\end{equation}
This ordering holds for every measurement value. For a nonlinear energy, the law of total covariance provides the corresponding average statement. Under $X\sim\pi_{\mathcal C_0}$ and the declared likelihood for future measurements $Y$,
\begin{equation}
\operatorname{Cov}(X)
=
\mathbb E_Y\!\left[\operatorname{Cov}(X\mid Y)\right]
+
\operatorname{Cov}_Y\!\left(\mathbb E[X\mid Y]\right).
\label{eq:app-d-total-covariance}
\end{equation}
Consequently, the expected conditional covariance is no larger than the original covariance in the positive-semidefinite order. This average is taken under the fixed model; the next subsection relates model-based credible sets to coverage under the true field law.

\subsection{Credible sets and expected coverage}
\label{app:coverage}

Generated ensembles represent uncertainty in quantities derived from complete fields. The following result separates the contribution of the physics-defined model from the accuracy of the implemented sampler and credible-set construction.

\subsubsection{Conditional laws and computational discrepancies}
\label{app:coverage-laws}

Let $\mathcal H$ contain the information fixed before a new measurement vector $Y$ is generated: calibration, any measurements in $\mathcal C_0$, the new observation operators and noise covariances, trained parameters, and sampling settings. All laws below are conditional on $\mathcal H$, suppressed in the notation. Let $\pi^\ast$ be the true conditional field law and $\pi_{\mathcal C_0}$ the fixed model. Assume that truth and model share the likelihood $\ell(\mathrm dy\mid x)$ specified by \eqref{eq:obs}, and define
\begin{equation}
P(\mathrm dx,\mathrm dy)
=
\pi^\ast(\mathrm dx)\ell(\mathrm dy\mid x),
\qquad
Q(\mathrm dx,\mathrm dy)
=
\pi_{\mathcal C_0}(\mathrm dx)\ell(\mathrm dy\mid x).
\label{eq:app-d-coverage-joint-laws}
\end{equation}
Write $P_Y$ and $Q_Y$ for the measurement marginals. For a measurable reported quantity $g:\mathbb R^d\to\mathbb R^k$, let $\mathcal T(x,y)=(g(x),y)$ and define the projected joint discrepancy
\begin{equation}
\mathcal D_g
=
\operatorname{KL}\!\left(
\mathcal T_\#P\,\|\,\mathcal T_\#Q
\right),
\label{eq:app-d-projected-discrepancy}
\end{equation}
where $\mathcal T_\#$ denotes the law induced by $\mathcal T$ and KL uses natural logarithms. The assumption $\mathcal D_g<\infty$ includes absolute continuity and integrability for these projected joint laws. It implies $P_Y\ll Q_Y$, so the exact-model conditionals are defined for $P_Y$-almost every measurement vector.

Let $P^{(g)}(\cdot\mid y)$, $Q^{(g)}(\cdot\mid y)$, and $Q_\theta^{(g)}(\cdot\mid y)$ be the true, exact-model, and implemented-sampler conditional laws of $g(X)$, respectively. The exact-model field law is $\pi_{\mathcal C}$ from \eqref{eq:cond}. The implemented law includes the numerical sampling procedure.

For $0<\alpha<1$, let $C_\alpha^{(g)}(Y,U)$ be a measurable random credible set, where $U$ contains the randomness used to construct it. Assume $U\perp X^\ast\mid(Y,\mathcal H)$, with $X^\ast$ drawn from the true field law. Using $\operatorname{TV}(\mu,\nu)=\sup_B|\mu(B)-\nu(B)|$ and $[r]_+=\max(r,0)$, define
\begin{equation}
\begin{aligned}
\epsilon_{\mathrm{samp}}
&=
\mathbb E_{P_Y}\operatorname{TV}\!\left(
Q_\theta^{(g)}(\cdot\mid Y),
Q^{(g)}(\cdot\mid Y)
\right),
\\
\epsilon_{\mathrm{set}}
&=
\left[
\alpha-\mathbb E_{P_Y,U}
Q_\theta^{(g)}\!\left(
C_\alpha^{(g)}(Y,U)\mid Y
\right)
\right]_+,
\\
\epsilon_{\mathrm{comp}}
&=
\epsilon_{\mathrm{samp}}+\epsilon_{\mathrm{set}}.
\end{aligned}
\label{eq:coverage-budgets}
\end{equation}
The sampler term measures its discrepancy from the exact model. The set term measures any deficit in expected set probability under the sampler. This probability is evaluated for a fresh draw independent of $U$ given $Y$, rather than by counting the samples used to construct their own set. Both terms average over the true measurement marginal.

\subsubsection{Expected coverage bound}
\label{app:expected-coverage}

\begin{theorem}[Expected functional coverage]
\label{thm:coverage}
Under the fixed-information setup above, the shared likelihood, and $U\perp X^\ast\mid(Y,\mathcal H)$, suppose $\mathcal D_g<\infty$. Then
\begin{equation}
\mathbb P\!\left\{
g(X^\ast)\in C_\alpha^{(g)}(Y,U)
\right\}
\ge
\alpha-\sqrt{\mathcal D_g/2}-\epsilon_{\mathrm{comp}}.
\label{eq:coverage}
\end{equation}
The probability averages over the true field, measurements, and set construction, conditional on $\mathcal H$.
\end{theorem}

\begin{proof}
The KL chain rule for the projected joint laws bounds their expected conditional KL by $\mathcal D_g$. Pinsker's and Jensen's inequalities~\citep{sason2016fdivergence} give
\begin{equation}
\begin{aligned}
&\mathbb E_{P_Y}
\operatorname{TV}\!\left(
P^{(g)}(\cdot\mid Y),Q^{(g)}(\cdot\mid Y)
\right)
\\
&\quad\le
\sqrt{\frac12\mathbb E_{P_Y}
\operatorname{KL}\!\left(
P^{(g)}(\cdot\mid Y)\,\|\,Q^{(g)}(\cdot\mid Y)
\right)}
\le
\sqrt{\mathcal D_g/2}.
\end{aligned}
\label{eq:app-d-coverage-tv-bound}
\end{equation}
Set $C=C_\alpha^{(g)}(Y,U)$. Conditional independence and two total-variation comparisons yield
\begin{equation}
\begin{aligned}
\mathbb P\{g(X^\ast)\in C\}
&=
\mathbb E_{P_Y,U}P^{(g)}(C\mid Y)
\\
&\ge
\mathbb E_{P_Y,U}Q_\theta^{(g)}(C\mid Y)
-\epsilon_{\mathrm{samp}}-\sqrt{\mathcal D_g/2}
\\
&\ge
\alpha-\epsilon_{\mathrm{set}}
-\epsilon_{\mathrm{samp}}-\sqrt{\mathcal D_g/2}.
\end{aligned}
\label{eq:app-d-coverage-comparison}
\end{equation}
The definition of $\epsilon_{\mathrm{comp}}$ completes the proof.
\end{proof}

The projected discrepancy allows the result to focus on the quantities being reported. When the field-level KL is finite, the shared likelihood, data processing, and the KL chain rule additionally give
\begin{equation}
\begin{aligned}
\mathcal D_g
&\le
\operatorname{KL}(\pi^\ast\|\pi_{\mathcal C_0}),
\\
\mathbb E_{P_Y}
\operatorname{KL}\!\left(
P(\cdot\mid Y)\,\|\,Q(\cdot\mid Y)
\right)
&=
\operatorname{KL}(\pi^\ast\|\pi_{\mathcal C_0})
-\operatorname{KL}(P_Y\|Q_Y)
\\
&\le
\operatorname{KL}(\pi^\ast\|\pi_{\mathcal C_0}).
\end{aligned}
\label{eq:posteriorbound}
\end{equation}
Here the conditional laws in the second line concern the full field. Theorem~\ref{thm:coverage} requires only finite $\mathcal D_g$, rather than finite field-level KL.

\paragraph{Interpretation.}
\label{app:coverage-support}
The theorem gives average coverage for a fixed model and observation likelihood, with explicit contributions from model discrepancy, sampling approximation, and set construction. Under correct specification, exact sampling, and expected set probability $\alpha$, expected true coverage equals $\alpha$ by conditioning on $Y$. With approximation, \eqref{eq:coverage} quantifies the corresponding lower bound. Its finite-discrepancy and support assumptions are essential, and its averaging does not assert the same coverage for every individual measurement outcome. These conditions align the result with the fixed-model conditioning and approximate sampling used by PhysDEM.

\section{Experimental details}
\label{app:experiments}

\subsection{Cases and data sources}
\label{app:exp-setup}

All cases use a $64\times64$ spatial grid, 48 time states, and 64-member ensembles. The baseline comparison uses $K=4$ wells; measurement reuse tests $K\in\{1,2,4,8\}$ with a fixed checkpoint per case. Synthetic reference fields are sampled by pCN from the declared model, and conditioning values include Gaussian noise.

\paragraph{Darcy flow.}
The unknowns are static log-permeability $g$ and pressure $u$:
\begin{equation}
c_D\partial_t u-\nabla\!\cdot\!\left(e^g\nabla u\right)-q=0,
\label{eq:exp-darcy}
\end{equation}
where $c_D=1$ is the storage coefficient and $q$ is a prescribed injection--extraction source. The synthetic case uses no-flow boundaries and zero initial pressure.

\paragraph{Advection--diffusion.}
The unknowns are static streamfunction $\psi$, source $s$, and concentration $c$:
\begin{equation}
\begin{aligned}
\partial_t c+\nabla\!\cdot(\mathbf v c)-D_c\Delta c+\kappa c-s&=0,\\
\mathbf v=\nabla^\perp\psi&=(-\partial_{\xi_2}\psi,\partial_{\xi_1}\psi).
\end{aligned}
\label{eq:exp-advdiff}
\end{equation}
The known diffusivity and decay rate are $D_c=2$ and $\kappa=0.5$. The synthetic case starts from zero concentration and uses zero exterior ghost-cell concentrations for both advective and diffusive boundary fluxes.

\paragraph{Fisher--KPP.}
Static log-diffusivity $g$ and state $u$ obey
\begin{equation}
\partial_t u-\nabla\!\cdot\!\left(e^g\nabla u\right)-r u(1-u)=0,
\label{eq:exp-fisher-kpp}
\end{equation}
with known growth rate $r=2$, no-flow boundaries, and an initial field formed from three Gaussian seeds clipped at one. Darcy and Fisher--KPP observe both $g$ and $u$; advection--diffusion observes $c$ only.

\paragraph{Discretization.}
Cell-centered finite volumes use harmonic face coefficients for heterogeneous diffusion and conservative centered advection fluxes. Darcy and advection--diffusion use backward Euler; Fisher--KPP treats diffusion implicitly and reaction explicitly. Known initial fields enter through additional residual rows.

\paragraph{SPE10.}
This case uses permeability structure from Model 2, layer 10 of the Tenth SPE Comparative Solution Project~\mbox{\citep{christie2001spe10}}. A window of the log-permeability field is resampled to the $64\times64$ inference grid, centered, and scaled by $0.5$; the static prior scales are informed by this selected field. The pressure reference is obtained by solving the Darcy equation~(\ref{eq:exp-darcy}) on a finer $256\times256$ grid and block-averaging onto the inference grid. Conditioning values are generated by adding Gaussian noise to log-permeability and pressure at the selected wells. This provides a controlled assessment of field recovery from scarce measurements in a heterogeneous reservoir benchmark.

\paragraph{Bemidji.}
This transport case incorporates field information from the USGS crude-oil spill research site near Bemidji, Minnesota~\mbox{\citep{trost2018bemidji}}. The data comprise 147 benzene records from 38 wells over seven campaigns between 1986 and 1995. Site geometry, well locations, campaign times, and a regional flow estimate inform the problem construction. The advection--diffusion equation~(\ref{eq:exp-advdiff}) uses site-specific coefficients, a known regional streamfunction plus a scaled unknown perturbation, and an unknown source field. The records inform the source and reference construction, and a forward solve generates the concentration reference. Conditioning values are synthetic noisy measurements at real well locations and campaign times, yielding 56 values at $K=8$. Comparison with the 147 records therefore assesses agreement with the source data; it is descriptive rather than an independent held-out validation.

\subsection{Baseline implementations and evaluation}
\label{app:exp-baselines}

The eight baselines are task-adapted implementations, summarized in Table~\ref{tab:exp-baseline-variants}. Their learned modules use 4,096 complete training fields and 512 validation fields per case, sampled approximately from the fixed physics-defined target. PhysDEM instead constructs training fields through internal approximate pCN pools. Fixed-reference comparisons share the reference and measurement configuration. The archived FNOPE implementation generates static parameters and completes states by a deterministic forward solve. These results characterize the evaluated implementations.

\begin{table}[t]
\centering
\caption{Implemented baseline variants.}
\label{tab:exp-baseline-variants}
\begin{tabularx}{\textwidth}{@{}l>{\raggedright\arraybackslash}X@{}}
\toprule
Method & Adaptation used in the comparison \\
\midrule
DiffusionPDE & Joint diffusion with observation and PDE guidance~\citep{huang2024diffusionpde}. \\
DPS & A learned joint diffusion prior with likelihood-gradient guidance~\citep{chung2023dps}. \\
DAPS & Decoupled annealing with diffusion denoising and Langevin likelihood updates~\citep{zhang2025daps}. \\
PnP-DM & Alternating diffusion-prior and Gaussian-likelihood sampling~\citep{wu2024pnpdm}. \\
FunDPS & Correlated-noise diffusion with observation guidance; PDE-guidance strength is validation-selected~\citep{yao2025fundps}. \\
PIDM & Physics-informed diffusion training with DPS observation guidance at inference~\citep{bastek2025physics}. \\
Di-BiLPS & Conditional encoding, a variational autoencoder, and latent diffusion~\citep{li2026dibilps}. \\
FNOPE & Conditional parameter generation followed by deterministic state completion~\citep{moss2025fnope}. \\
\bottomrule
\end{tabularx}
\end{table}

\begin{figure}[t]
\centering
\includegraphics[width=\textwidth]{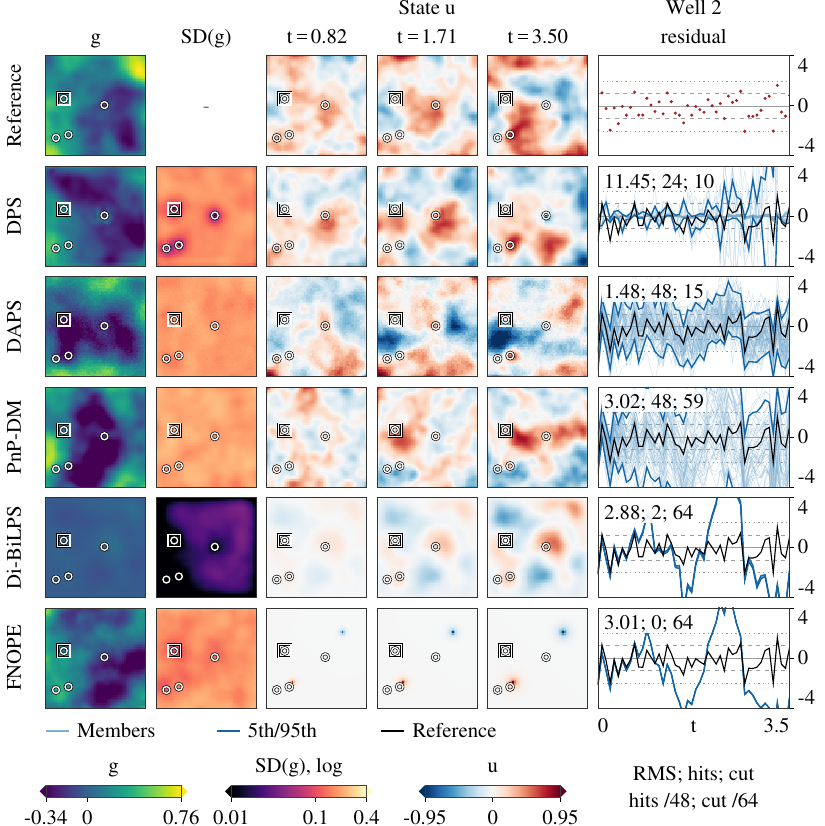}
\caption{Additional Darcy baseline comparisons at $K=4$ with 64-member ensembles. Reference fields, observations, sample selection, and plotting conventions match Figure~\ref{fig:darcy-main}; rows contain the five remaining methods. FNOPE pressure fields are completed by its deterministic forward solver.}
\label{fig:darcy-baselines-appendix}
\end{figure}

\paragraph{Reported diagnostics.}
\label{app:exp-metrics}
For $N=64$ generated fields $x^{(n)}$, the stacked observation model $(A,y,\Gamma)$ gives
\begin{equation}
z_{\mathrm{obs}}
=\left[\frac{1}{N n_y}\sum_{n=1}^{N}
(Ax^{(n)}-y)^\top\Gamma^{-1}(Ax^{(n)}-y)\right]^{1/2},
\qquad n_y=\dim(y).
\label{eq:exp-obs-diagnostic}
\end{equation}
This averages squared standardized errors over members and scalar observations. The residual diagnostic is
\begin{equation}
r_{\mathrm{PDE}}
=\frac{\operatorname{median}_{n}\|\mathcal R(x^{(n)})\|_W^2}{\varepsilon},
\qquad \varepsilon=0.09^2.
\label{eq:exp-residual-diagnostic}
\end{equation}
Here $\mathcal R$ includes the selected, rescaled Fourier residuals and initial-condition rows; the even-ensemble median uses the lower central value. One denotes the noise or prescribed residual scale, not a universal optimum. In particular, a zero residual from deterministic state completion does not establish joint posterior calibration.

\paragraph{Mean-rank checks.}
Following simulation-based calibration~\citep{talts2020sbc}, each archived evaluation uses $M=48$ truths and $N=64$ conditional samples. For statistic $T_j$, define
\begin{equation}
R_{ij}=\sum_{n=1}^{N}\mathbf{1}\{T_j(x_i^{(n)})<T_j(x_i^\star)\},
\quad
\bar R_j=\frac{1}{M}\sum_{i=1}^{M}R_{ij},
\quad
z_j=\frac{\bar R_j-N/2}{\sqrt{v_R/M}}.
\label{eq:exp-sbc-rank}
\end{equation}
Archived conventions use $v_R=N^2/12$ for legacy evaluations or $v_R=N(N+2)/12$ with randomized tie handling for common-v2 evaluations. For Darcy, Fisher--KPP, and SPE10, the five statistics are $\overline g$, $\overline{g^2}$, $\overline{u_{24}}$, $\overline{u_{47}}$, and $\overline{u^2}$. The transport cases use seven: $\overline\psi$, $\overline{\psi^2}$, $\overline s$, $\overline{s^2}$, $\overline{c_{24}}$, $\overline{c_{47}}$, and $\overline{c^2}$. Bars denote spatial averages, except the last statistic, which averages over space and time; frame indices are zero-based. Table~\ref{tab:baselines} counts $|z_j|<2.6$ for five statistics or $|z_j|<2.7$ for seven. These are mean-rank checks, not full rank-uniformity tests. Truth pools and rank conventions differ across archives, so comparisons are descriptive rather than paired.

\paragraph{Reading the Darcy comparisons.}
\label{app:exp-baseline-results}
Figure~\ref{fig:darcy-main} and Figure~\ref{fig:darcy-baselines-appendix} show each method's first saved sample, static-field ensemble standard deviation on a logarithmic scale, and pressure at $t=0.82,1.71,3.50$. Circles mark wells; the square selects well 2. Thin blue curves show standardized observation residuals, thick blue curves their pointwise 5th/95th percentiles, and black the reference residual. The annotations give residual RMS, the number of 48 times the interval contains the reference, and the number of 64 trajectories exceeding the displayed $\pm4$ range. Statistics use unclipped values; shared color limits may saturate. The hit count describes one well and reference, not simultaneous-trajectory coverage.

\paragraph{Reading the Bemidji comparison.}
\label{app:exp-bemidji}
Figure~\ref{fig:bemidji-site-informed} displays the 1995 mean and standard deviation after transforming concentrations to log scale. Model values are floored at $10^{-6}$\,mg/L before transformation; recorded values use $10^{-4}$\,mg/L. Nonpositive model values map to the plotting floor. The pointwise 5th/95th percentiles describe latent-field ensembles without added observation noise and contain 56 of the 147 source records. As explained in Appendix~\ref{app:exp-setup}, these records also inform reference construction.

\begin{figure}[t]
\centering
\includegraphics[width=\textwidth]{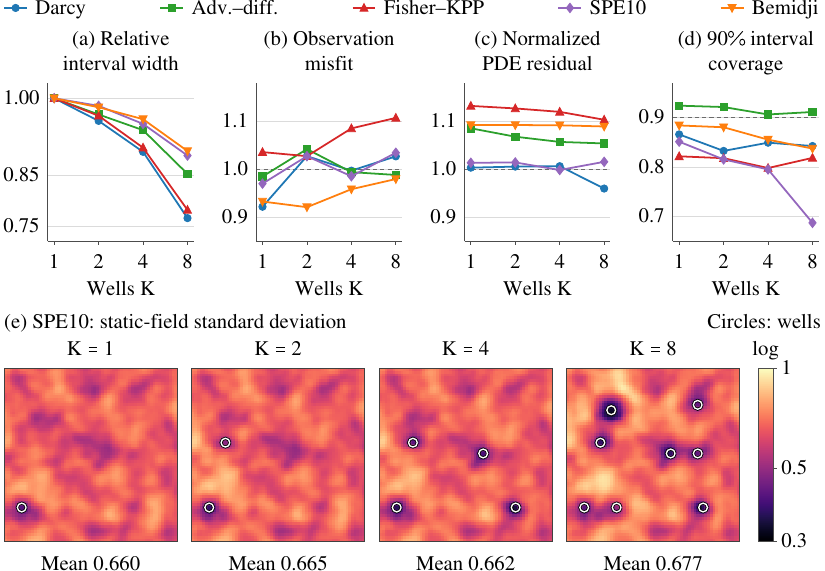}
\caption{Measurement reuse with SPE10 static-field uncertainty. Panels (a--d) repeat Figure~\ref{fig:measurement-reuse}; panel (e) shows the ensemble standard deviation of log permeability with a common logarithmic color scale. Circles mark measured wells, and labels give spatial means.}
\label{fig:measurement-reuse-spe10-appendix}
\end{figure}

\subsection{Reuse across measurement counts}
\label{app:exp-measurement-reuse}

Each case reuses one trained checkpoint and the same reference field for $K\in\{1,2,4,8\}$, drawing 64 samples per configuration. Measurement conditioning updates the augmented Gaussian reference while the denoiser weights and physics remain fixed. Well locations are nested, but dynamic observation noise is regenerated across $K$; this is a comparison of configurations, not sequential addition of observations with all previous values retained. Darcy, Fisher--KPP, and SPE10 use $49K$ scalar observations (one static and 48 dynamic values per well); advection--diffusion uses $48K$ and Bemidji $7K$.

\paragraph{Interval and spread diagnostics.}
Let $L_{K,j}$ and $U_{K,j}$ be the empirical 5th and 95th percentiles across the 64 samples at coordinate $j$, using linear interpolation. For the complete field vector of dimension $d$,
\begin{equation}
W_K=\frac{1}{d}\sum_{j=1}^{d}(U_{K,j}-L_{K,j}),
\qquad
C_K=\frac{1}{d}\sum_{j=1}^{d}
\mathbf{1}\{L_{K,j}\leq x_j^\star\leq U_{K,j}\}.
\label{eq:exp-reuse-intervals}
\end{equation}
Figure~\ref{fig:measurement-reuse}(a,d) reports $W_K/W_1$ and $C_K$. All static-field coordinates and all 48 dynamic frames enter with equal coordinate weights, including unobserved positions. Thus $C_K$ is a fixed-reference inclusion fraction, not repeated-data coverage or SBC. Observation and residual diagnostics follow Appendix~\ref{app:exp-metrics}. The maps show pointwise static-field sample standard deviations with denominator $N-1$; their annotations are spatial means.

\paragraph{Complementary SPE10 results.}
Figure~\ref{fig:measurement-reuse-spe10-appendix} supplements the Fisher--KPP maps. SPE10's mean static-field standard deviation changes from $0.660$ to $0.677$ between $K=1$ and $K=8$, while its full-field interval width decreases by $11.2\%$ and $C_K$ falls from $0.851$ to $0.687$. Interval contraction therefore does not imply improved calibration or shrinking uncertainty in every field component.

\paragraph{Weakly constrained static fields.}
\label{app:exp-static-identifiability}
Advection--diffusion and Bemidji observe concentration only. At $K=4$, the ensemble-mean relative $L^2$ errors for the streamfunction $\psi$ are $1.028$ and $1.019$, and its spread--skill ratios are $0.938$ and $1.162$, respectively. The relative error is the Euclidean reconstruction error divided by the reference norm; SSR follows equation~(\ref{eq:exp-ablation-ssr}) with $g$ replaced by $\psi$. From $K=1$ to $K=8$, the spatial RMS ensemble spread changes by less than $0.2\%$ in both cases. These results are consistent with weak constraints on $\psi$; they do not establish that its posterior equals its prior.

\subsection{Component ablations and tolerance sensitivity}
\label{app:ablation}

Table~\ref{tab:ablation} uses Darcy at $K=4$ with 64-member ensembles. The training variants disable reference augmentation or replace $(X-m_a)/s$ by $(X-m_a)/\tau$, where $\tau=1-a^2$, retaining the main calibration. In the available implementation, the augmentation flag affects training references, whereas the evaluation scripts still construct augmented references. The no-augmentation result therefore also reflects a training--inference reference mismatch and does not isolate end-to-end removal of augmentation. The reference-only variant draws directly from the augmented Gaussian without learned transport; its SBC entry was not evaluated.

\paragraph{Diagnostics.}
SBC is $\max_j|z_j|$ over the five field statistics defined in Appendix~\ref{app:exp-metrics}, evaluated on 48 truths (pass if below $2.6$). A dash means not evaluated. Other diagnostics use a fixed reference field per run. With $\bar g=N^{-1}\sum_n g^{(n)}$, $N=64$, and $d_g=4096$, the archived spread--skill ratio is
\begin{equation}
\mathrm{SSR}_g=
\frac{\displaystyle
\sqrt{\frac{N+1}{N}\frac{1}{d_g(N-1)}
\sum_{j=1}^{d_g}\sum_{n=1}^{N}(g_j^{(n)}-\bar g_j)^2}}
{\displaystyle
\sqrt{\frac{1}{d_g}\sum_{j=1}^{d_g}(\bar g_j-g_j^\star)^2}}.
\label{eq:exp-ablation-ssr}
\end{equation}
One is a reference spread--error scale; moving closer to it does not establish improvement in every diagnostic.

\paragraph{Tolerance sensitivity.}
The $0.05$ and $0.15$ variants recalibrate and retrain, changing the target distribution. All ratios in Table~\ref{tab:ablation} divide by the common main tolerance $\varepsilon=0.09^2=0.0081$, yielding $0.35$ and $2.10$; using each variant's own squared tolerance would instead give approximately $1.12$ and $0.76$. The $0.15$ setting has $\mathrm{SSR}_g=1.08$, closer to one than the full model's $0.89$, but its mean-rank statistic is $5.20$. These results illustrate trade-offs across diagnostics rather than a uniformly optimal setting.

\subsection{Sampling cost and noise sensitivity}
\label{app:exp-cost-noise}

\paragraph{Archived sampling cost.}
Figure~\ref{fig:sampling-cost} uses elapsed batch time divided by the number of generated members, not single-sample latency. The PhysDEM point uses 64 samples in network minibatches of two, EMA weights, 150 Heun steps (300 network evaluations), and $a_{\max}=0.995$. Its $1.47$\,s/member includes sampling, summary-statistic computation, and transfer of those statistics to the CPU; reference construction, residual scoring, and offline training are outside the timer. DiffusionPDE takes $4.38$\,s/member in its archived run. Batch sizes and CUDA-graph settings differ across methods; DiffusionPDE and Di-BiLPS times come from imported run records. The comparison therefore characterizes these implementations, not a controlled end-to-end hardware benchmark.

\paragraph{Observation-noise variance.}
Figure~\ref{fig:noise-sensitivity} uses $K=4$ and 64 samples per setting. Each case reuses a fixed checkpoint while the observation-noise covariance is multiplied by $0.1$, $1$, or $10$. Gaussian conditioning and simulated measurement noise use the corresponding covariance; these are variance changes, not a test of misspecified noise. Across the 15 settings, $z_{\mathrm{obs}}$ ranges from $0.906$ to $1.146$ and $r_{\mathrm{PDE}}$ from $0.974$ to $1.120$, using the definitions in Appendix~\ref{app:exp-metrics}. The dashed unit lines mark reference scales. These results support stability of the reported diagnostics under the tested noise changes, without establishing full posterior calibration.

\end{document}